\documentclass[10pt,twocolumn,letterpaper]{article}

\usepackage[pagenumbers]{cvpr} 

\usepackage[pagebackref,breaklinks,colorlinks,citecolor=cvprblue]{hyperref}
\usepackage[utf8]{inputenc} 
\usepackage[T1]{fontenc}    
\usepackage{amsfonts}       
\usepackage{nicefrac}       
\usepackage{microtype}      
\usepackage{multirow}
\usepackage{bm}
\usepackage{listings}

\usepackage[dvipsnames]{xcolor}

\usepackage{mathtools}
\usepackage{amsthm}

\usepackage{algpseudocode}
\usepackage{algorithm}

\theoremstyle{plain}

\newtheorem{proposition}{Proposition}

\theoremstyle{definition}

\theoremstyle{remark}

\usepackage{wrapfig}
\usepackage{varwidth}
\usepackage{arydshln}
\usepackage{subcaption}

\definecolor{ered}{rgb}{0.72, 0.16, 0.2}
\definecolor{limegreen}{rgb}{0.2, 0.8, 0.2}
\definecolor{cvprblue}{rgb}{0.21,0.49,0.74}
\definecolor{paramred}{RGB}{205, 95, 95}

\newcommand{\q}[1]{``#1''} 

\newcommand{\rankdown}[1]{\scriptsize\textcolor{ered}{\boldmath $(\!#1_{\downarrow}\!)$}}
\newcommand{\rankup}[1]{\scriptsize\textcolor{limegreen}{\boldmath$(\!#1^{\uparrow}\!)$}}

\newcommand{\paramstimes}[1]{\scriptsize\textcolor{ered}{\boldmath $(\times #1)$}}

\newcommand*\samethanks[1][\value{footnote}]{\footnotemark[#1]}

\AtEndPreamble{
    \usepackage[capitalize]{cleveref}
    \crefname{section}{Sec.}{Secs.}
    \Crefname{section}{Section}{Sections}
    \Crefname{table}{Table}{Tables}
    \crefname{table}{Tab.}{Tabs.}
    \Crefname{equation}{Eq.}{Eqs.}
    \Crefname{figure}{Fig.}{Figs.}
    \Crefname{tabular}{Tab.}{Tabs.}
    \Crefname{algorithm}{Alg.}{Algs.}
    \Crefname{proposition}{Prop.}{Props.}
    \Crefname{appendix}{App.}{Apps.}
    \creflabelformat{equation}{#2\textup{#1}#3}
}

\title{Learning on Model Weights using Tree Experts}

\author{Eliahu Horwitz\thanks{Equal contribution} \qquad Bar Cavia\samethanks \qquad Jonathan Kahana\samethanks \qquad Yedid Hoshen \\
The Hebrew University of Jerusalem, Israel\\
  \small\url{https://horwitz.ai/probex/}\\
\small\texttt{\{eliahu.horwitz, bar.cavia, jonathan.kahana, yedid.hoshen\}@mail.huji.ac.il} \\
}

\begin{document}
\maketitle

\begin{abstract} 
The number of publicly available models is rapidly increasing, yet most remain undocumented. Users looking for suitable models for their tasks must first determine what each model does. Training machine learning models to infer missing documentation directly from model weights is challenging, as these weights often contain significant variation unrelated to model functionality (denoted nuisance). Here, we identify a key property of real-world models: most public models belong to a small set of \textit{Model Trees}, where all models within a tree are fine-tuned from a common ancestor (e.g., a foundation model). Importantly, we find that within each tree there is less nuisance variation between models. Concretely, while learning across Model Trees requires complex architectures, even a linear classifier trained on a single model layer often works within trees. While effective, these linear classifiers are computationally expensive, especially when dealing with larger models that have many parameters. To address this, we introduce \textit{Probing Experts} (ProbeX), a theoretically motivated and lightweight method. Notably, ProbeX is the first probing method specifically designed to learn from the weights of a single hidden model layer. We demonstrate the effectiveness of ProbeX by predicting the categories in a model's training dataset based only on its weights. Excitingly, ProbeX can map the weights of Stable Diffusion into a weight-language embedding space, enabling model search via text, i.e., zero-shot model classification.
\end{abstract}

\section{Introduction}
\label{sec:intro}
In recent years, the number of publicly available neural network models has skyrocketed, with over one million models now hosted on Hugging Face. 
Ideally, this abundance would allow users to simply download the most suitable model for their task, thereby saving resources, reducing training time, and potentially improving accuracy. However, the lack of adequate documentation for most models makes it challenging for users to determine a model’s suitability for specific tasks. This motivates developing machine learning methods that can infer model functionality and missing documentation directly from model weights. The emerging field of weight-space learning \citep{schurholt2021self,navon2023equivariant,neural_functional_transformers,zhou2024universal,pal2024model} studies how to design and train metanetworks, neural networks that take the weights of other neural networks as inputs (see \cref{fig:weight_space_learning}). Previous works learned metanetworks that predict training data attributes \citep{sane,lim2023graph}, model performance \citep{navon2023equivariant}, and even generate new model parameters \citep{weights2weights,peebles2022learning}. In this work, we focus on predicting the categories in a model's training dataset as a proxy for the concepts it can recognize or generate. As a initial step towards model search-by-text, we demonstrate that it is sometimes possible to align weights with language, enabling zero-shot model classification.

\begin{figure}[t]
\centering
\includegraphics[width=0.95\linewidth]{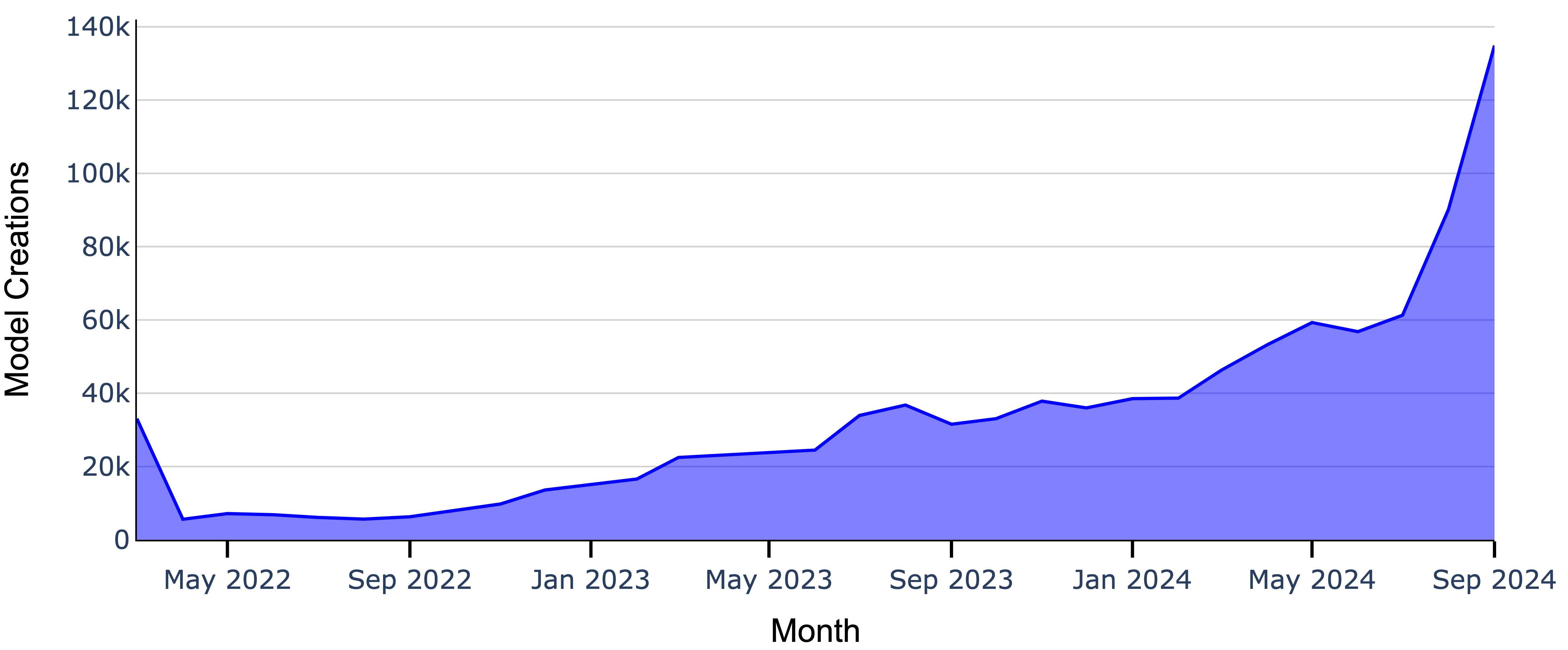} \\
\vspace{-5pt}
 \caption{\textit{\textbf{Growth in Hugging Face models:}} 
 The number of public models is growing fast, but they are mostly undocumented. Fully benefiting from them requires effective model search.}
 \vspace{-15pt}
 \label{fig:hf_growth}
\end{figure}

\begin{figure*}[t]
\centering
\includegraphics[width=0.95\linewidth]{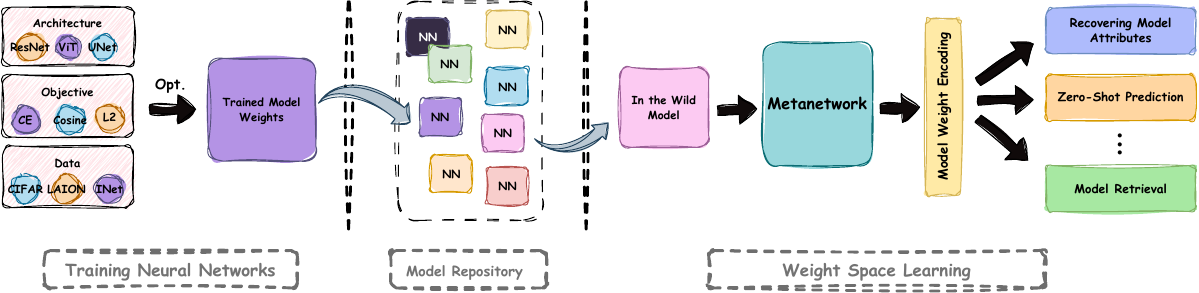}
\vspace{-5pt}
 \caption{\textit{\textbf{Weight-space learning:}} 
 \textit{Left.} Model weights are a direct product of the optimization process and training data. \textit{Center.} Models are uploaded to public repositories, e.g., Hugging Face, typically lacking documentation. This prevents model search, making models hard to discover and reuse. \textit{Right.} Weight-space learning treats each model as a data point and designs metanetworks: networks that process weights of other models as inputs. We train metanetworks that predict categories in a model's training dataset. As a first step towards model search-by-text, we also align weights with language, enabling zero-shot model classification.
 }
 \vspace{-10pt}
\label{fig:weight_space_learning}
\end{figure*}

However, extracting meaningful information from model weights is challenging. While the weights of a neural network are a function of its training data, they are also affected by the optimization process, which may introduce nuisance variation unrelated to attributes of interest. Neuron permutation \citep{hecht1990algebraic} is perhaps the most studied nuisance factor and has driven research into permutation-invariant architectures \citep{lim2023graph,navon2023equivariant_alignment,navon2023equivariant,kofinas2024graph,neural_functional_transformers} and carefully designed augmentations \citep{sane,schurholt2021self,schurholt2022hyper,schmidhuber}. In this paper, we highlight another important nuisance factor, the weights at the beginning of optimization.

The key insight in this paper is that models within a Model Tree \citep{mother} have reduced nuisance variation. A Model Tree describes a set of models that share a common ancestor (root), with each model derived by fine-tuning either from the root or from one of its descendants. For example, the Llama3 \citep{dubey2024llama} Model Tree includes all models fine-tuned from Llama3 or any of its descendants. In practice, most public models belong to a relatively small number of Model Trees, for instance, on Hugging Face, fewer than $20$ Model Trees cover most models (see \cref{fig:hf_breakdown}). \textit{We hypothesize that learning on models from the same tree is significantly simpler than learning from models across different trees.} We demonstrate this empirically, showing that standard linear models perform well on models within the same tree but fail to learn when applied to models from many different trees.

However, standard linear classifiers require too many parameters for learning on large models. To address this, we present single layer \textit{\textbf{Probing} \textbf{eX}perts} (ProbeX), a theoretically grounded architecture that scales weight-space learning to large models. Unlike conventional probing methods, ProbeX operates on hidden model layers. Remarkably, ProbeX can handle models with hundreds of millions of parameters, requiring under $10$ minutes to train. When working with multiple Model Trees, we use a Mixture-of-Experts.

To evaluate our method, we introduce a dataset of $14{,}000$ models across $5$ disjoint Model Trees spanning multiple architectures and functionalities. ProbeX achieves state-of-the-art results on the task of training category prediction, accurately identifying the specific classes within a model’s training dataset. Excitingly, ProbeX can also align fine-tuned Stable Diffusion weights with language representations. This capability enables a new task: zero-shot model classification, where models are classified via a text prompt describing their training data, ProbeX achieves $89.8\%$ accuracy on this.

To summarize, our main contributions are:
\begin{enumerate}
    \item Identifying that learning within Model Trees is much simpler than learning across trees. 
    \item Introducing Probing Experts (ProbeX), a lightweight, theoretically motivated method for weight-space learning.
    \item Proposing the task of zero-shot model weight classification and tackling it by aligning model weights with language representations, achieving strong results.
\end{enumerate}

\begin{figure*}[t]
    \centering
    \begin{tabular}{cccc}
        \begin{minipage}{0.21\textwidth}
            \centering
            \includegraphics[width=\textwidth]{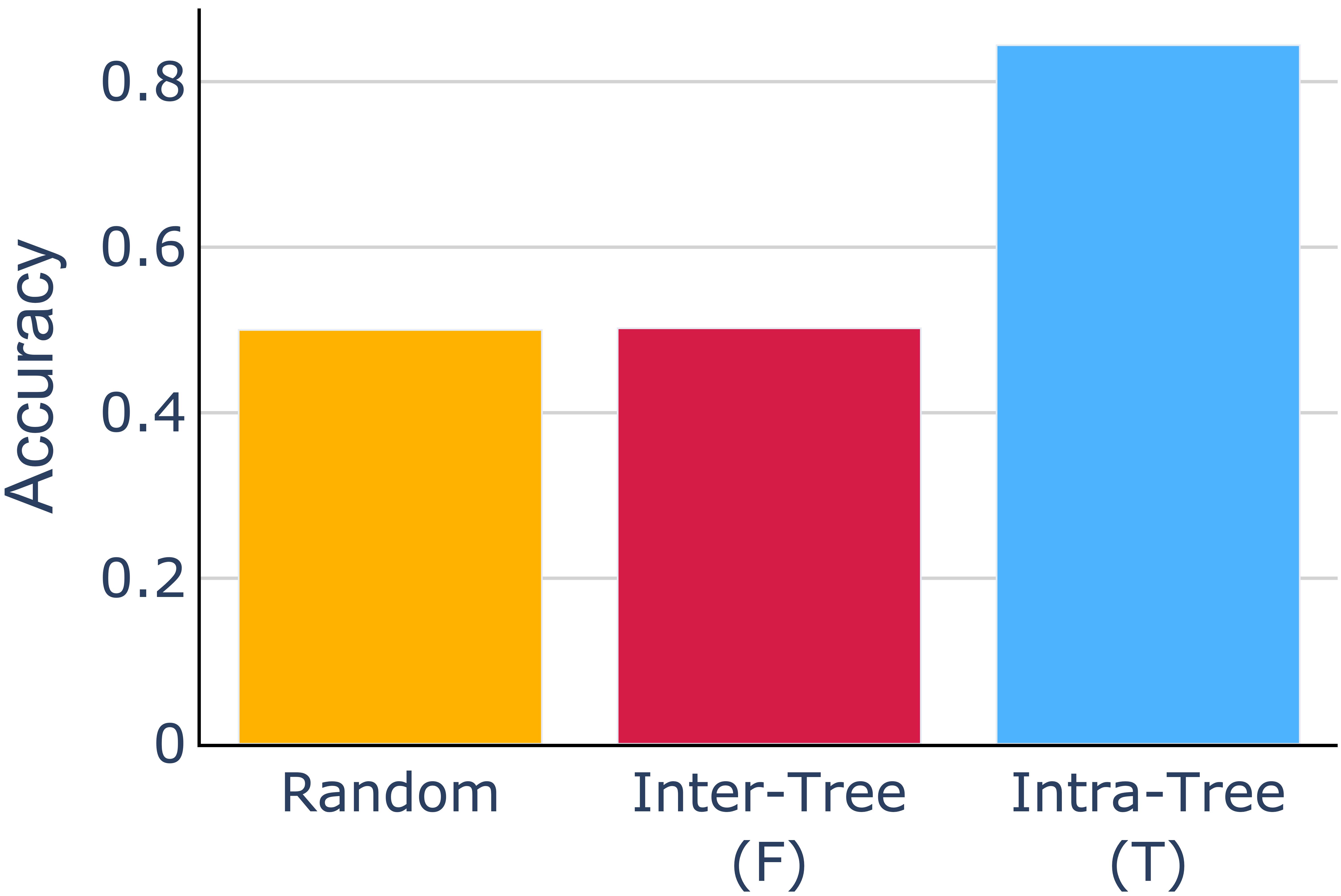}
            \subcaption{\textbf{\textit{Intra vs. Inter tree learning}}}
            \label{fig:resnet9_t_f}
        \end{minipage} &
        
        \begin{minipage}{0.21\textwidth}
            \centering
            \includegraphics[width=\textwidth]{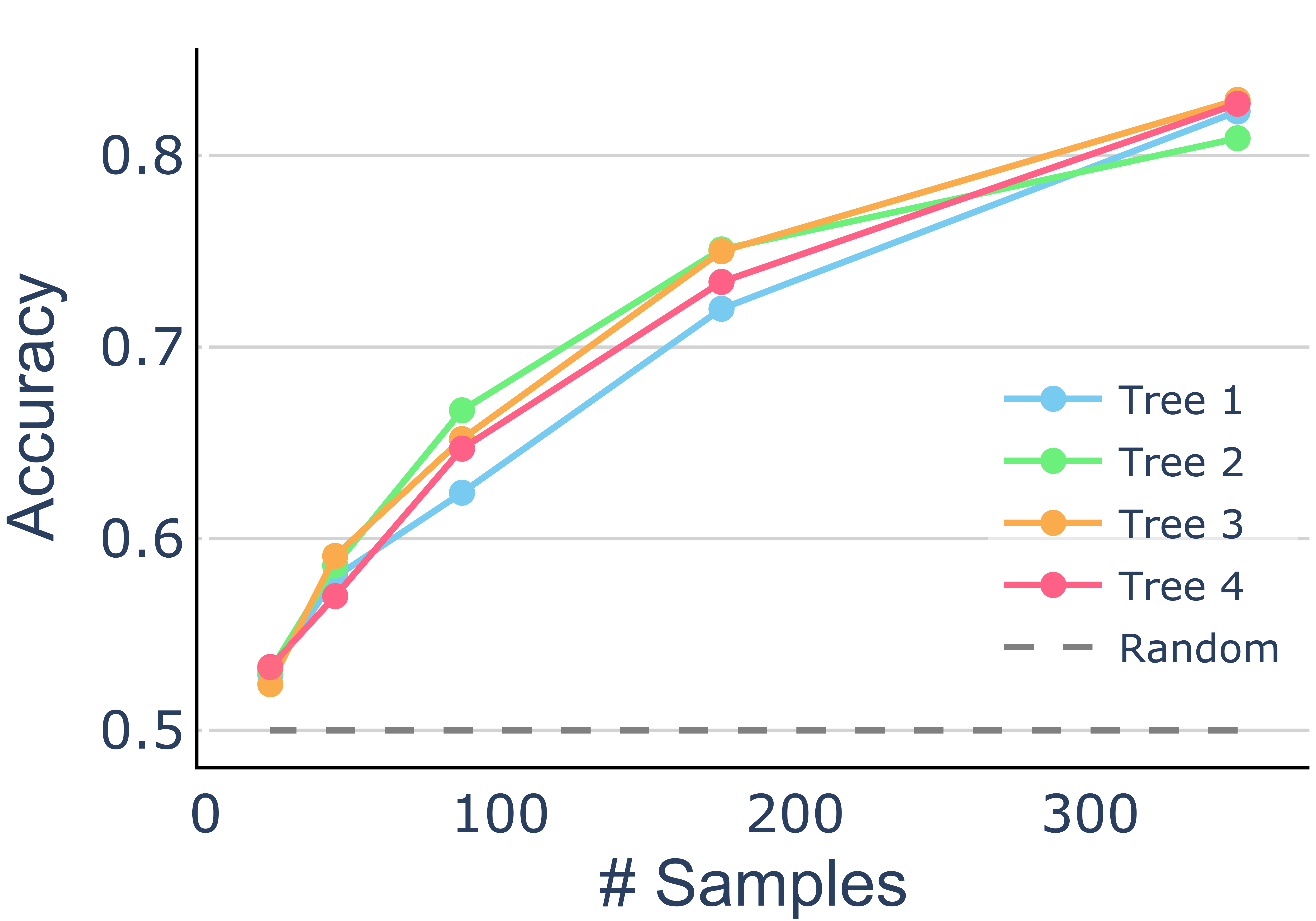}
            \subcaption{\textbf{\textit{Positive transfer within trees}}}
            \label{fig:intra_train_size}
        \end{minipage} &
        
        \begin{minipage}{0.26\textwidth}
            \centering \subcaption{\textbf{\textit{Negative transfer between trees}}}
            \begin{tabular}{ccc}
                \textbf{\# Trees} & \textbf{\# Models} & \textbf{Acc.} \\
                \toprule
                1 & 350  & 0.844 \\
                2 & 700  & 0.752 \\
                3 & 1050 & 0.686 \\
                4 & 1400 & 0.724 \\
            \end{tabular}
            \label{tab:negative_transfer}
        \end{minipage} &
        \begin{minipage}{0.21\textwidth}
            \centering
            \includegraphics[width=\textwidth]{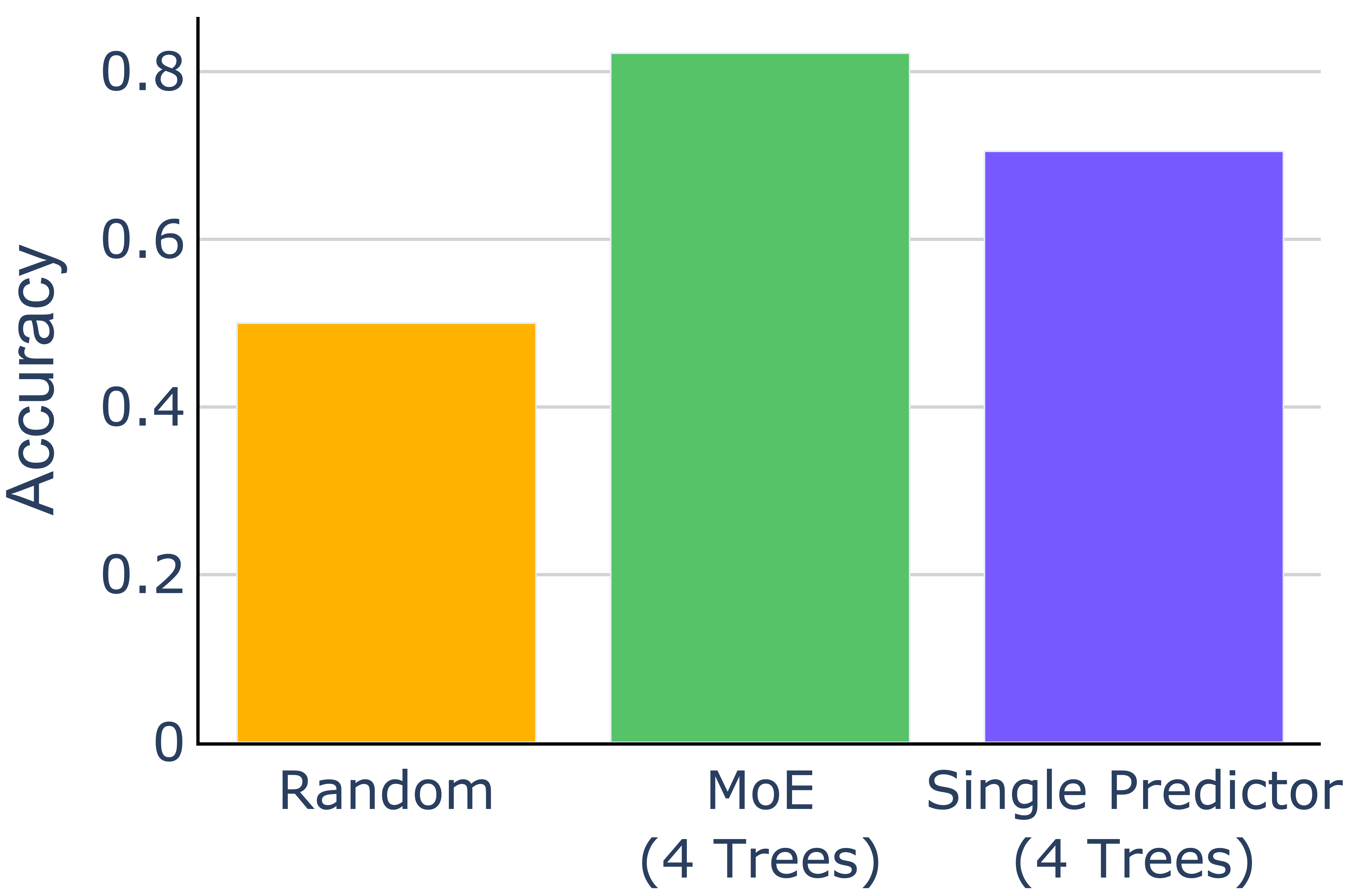}
            \subcaption{\textbf{\textit{MoE vs. Shared predictor}}}
            \label{fig:moe_vs_joint}
        \end{minipage}
    \end{tabular}
    \caption{\textbf{\textit{Tree membership and weight-space learning:}} We conduct 4 motivating experiments that illustrate the benefits of learning within Model Trees. In each experiment, we train a linear classifier to predict the classes a ViT model was fine-tuned on. First, we show that \textit{learning within Model Trees is significantly simpler (a)} by comparing a metanetwork trained on models from the same tree ($T$) with one trained on models from different trees ($F$). Next, we demonstrate \textit{positive transfer within the same tree (b)} by showing that adding more models from the same tree improves the performance.  Surprisingly, we observe \textit{negative transfer between Model Trees (c)}, where adding samples from other trees degrades performance on a single tree. Finally, we find that \textit{expert learning is preferable when learning from multiple trees (d)}, as a single shared metanetwork performs worse than an expert metanetwork per tree (MoE).}
\end{figure*}

\section{Related works}
\label{sec:related_works}
Despite the popularity of neural networks, little research has explored using their weights as inputs for machine learning methods. \citet{statnn,nws} were among the first to systematically analyze model weights to predict undocumented properties like the training dataset or generalization error. Some works aim to learn general representations \citep{sane,schurholt2021self,schurholt2022hyper,schmidhuber} for multiple properties, while others incorporate specific priors \citep{lim2023graph,navon2023equivariant_alignment,navon2023equivariant,kofinas2024graph,neural_functional_transformers,dsire} to directly predict the property. A major challenge with model weights is the presence of many parameter space symmetries \citep{hecht1990algebraic}. For instance, permuting neurons in hidden layers of an MLP doesn't change the network output. Thus, neural networks designed to take weights as inputs need to account for these symmetries. To avoid the issue of weight symmetries, recent methods \citep{probegen,kofinas2024graph,schmidhuber,cunningham2023sparse,lu2023content} propose using \textit{probing}. In this approach, a set of probes are optimized to serve as inputs to the model and the outputs act as the model representation. Probes are also used in other fields such as \citep{huang2024lg,tahan2024label,bau2017network,shaham2024multimodal,dravid2023rosetta} mechanistic interpretability for solving different tasks. However, until now, this was limited to passing the probes through the entire model and did not apply to single layers. Concurrent to our work, \citet{putterman2024learning} propose a method for linearly classifying LoRA \citep{lora} models to infer their functionality. While they do not describe it this way, it can be viewed as probing.

Other applications include generating weights \citep{ha2016hypernetworks,nern,peebles2022learning,erkocc2023hyperdiffusion,weights2weights}, merging models \citep{ties_merging,shah2023ziplora,ilharco2022editing,wortsman2022model,ortiz2023task}, and recovering the weights of unpublished models \citep{spectral_detuning, carlini2024stealing}.

\section{Motivation}
\label{sec:motivation}

\subsection{The challenge}
While machine learning on images, text, and audio is fairly advanced, learning from model weights is still in its infancy and the key nuisance factors remain unclear. Many approaches focused on neuron permutations \citep{navon2023equivariant,zhou2024permutation,kofinas2024graph} as the core nuisance factor. However, permutations are not likely to describe all nuisance variation, as neurons and layers can serve different roles across models and architectures. This paper highlights that learning within Model Trees \citep{mother} reduces nuisance variation, making learning simpler.

\subsection{Seeing the forest by seeing the trees}
\textbf{Background: Model Trees.} Following \citet{mother}, we represent model populations as a \textit{Model Graph} comprising disjoint directed \textit{Model Trees}. In this graph, each node is a model, with directed edges connecting each model to those directly fine-tuned from it. Since a model has at most one parent, the graph forms a set of non-overlapping trees. Importantly, while all the models within a Model Tree share the same architecture, two models with the same architecture but different roots belong to different Model Trees. E.g., DINO \citep{dino} and MAE \citep{mae} both use the same ViT-B/16 \citep{vit} architecture but form disjoint trees. Note that while Model Trees were originally proposed for model attribution, here, we use them differently: to group models with shared initial weights, thereby reducing nuisance variation. Consequently, we do not need to recover the precise tree structure, but only to map each model into a particular tree.

\noindent \textbf{Tree membership and weight-space learning.} Current weight-space methods generally rely on a single metanetwork to learn from a diverse model population spanning multiple trees. We hypothesize that learning within a single Model Tree is significantly simpler than learning across multiple Model Trees. I.e., we expect that dividing the population into distinct Model Trees and learning within each tree, can greatly simplify weight-space learning. To test this hypothesis, we simulate various model populations.

First, we create dataset $A$ by randomly selecting $50$ classes from CIFAR100 and dataset $B$ by randomly choosing $25$ of the remaining classes. Second, we \textit{pre-train} a classifier on $B$ for a \textit{single epoch}. We then train two different model populations, $T$ (Model Tree) and $F$ (Model Forest), each with $500$ ResNet9 models. All models are trained to classify between $25$ randomly selected classes from $A$. The populations differ in one aspect only: models in $F$ are initialized randomly, while models in $T$ are all initialized from the same model pre-trained on $B$. Therefore, \textit{all models in $T$ belong to the same Model Tree, while each model in $F$ belongs to a different tree.} Given a model, the task is to predict which $25$ out of the $50$ classes from $A$ it was trained on. Using $T$ and $F$, we can analyze learning within and across Model Trees.

\textit{\textbf{Is learning within the same Model Tree beneficial?}} We begin with a simple experiment, training a \textit{linear} metanetwork for models in $T$ and another one for models in $F$.
In line with our hypothesis, there is a large performance gap between the two settings (see \cref{fig:resnet9_t_f}). While learning on models within the same tree achieved good results ($0.844$), learning on models from many different trees achieved near random accuracy ($0.502$). This demonstrates that i) Model Tree membership introduces significant non-semantic variations in model weights, and ii) even a single epoch of shared pre-training might be enough to eliminate the variation.
 
 \begin{figure*}[t]
\centering
 \includegraphics[width=\linewidth]{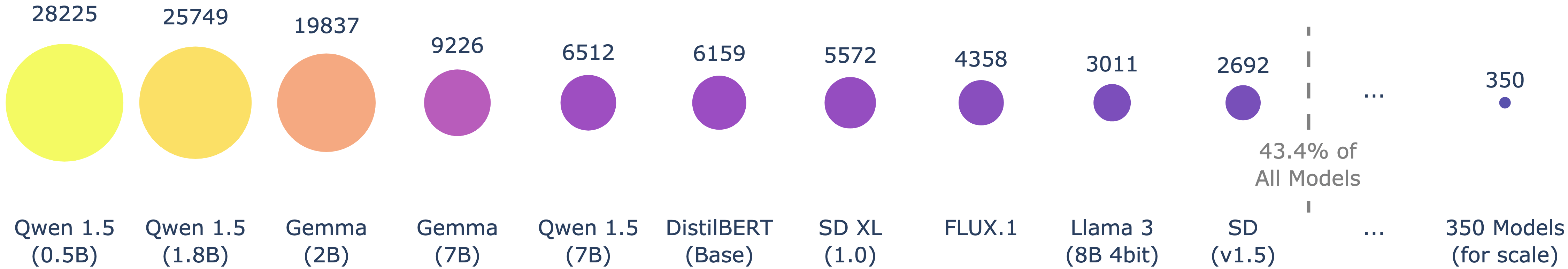} \\
 \caption{\textit{\textbf{Largest Model Trees on Hugging Face:}} We show the 10 largest Model Trees on Hugging Face. Our insight is that learning an expert for each tree greatly simplifies weight-space learning. This is a practical setting as a few large Model Trees dominate the landscape.} 
 \label{fig:hf_breakdown}
\end{figure*}

\textit{\textbf{Which models have positive transfer?}} Next, we investigate whether models have positive transfer. I.e., whether increasing the size of a training set helps learning. To this end, we pre-train $4$ different models on $B$ and use them to fine-tune populations $T_1, \ldots, T_4$ similarly to the way $T$ was created. First, we train different metanetworks on increasingly larger population subsets. We find that when all models belong to the same tree, increasing the size of the training set results in better performance (see \cref{fig:intra_train_size}). I.e., models in the same tree have positive transfer. 

Next, we test whether adding models from different trees is helpful. We start by training and evaluating a metanetwork on models from $T_1$. We gradually add to the training set models from other trees ($T_2,\ldots, T_4$) and check whether training different metanetworks on these larger datasets improves the classification of models from $T_1$. Surprisingly, we find that adding models from different trees \textit{decreases} the accuracy on $T_1$ (see \cref{tab:negative_transfer}), demonstrating that learning from multiple trees has a negative transfer effect.  

\textit{\textbf{How to learn from multiple Model Trees?}} Finally, we compare learning a separate expert model per tree and combine them via a Mixture-of-Experts (MoE) approach vs. learning a single shared metanetwork for all trees. We find that the MoE approach outperforms joint training (see \cref{fig:moe_vs_joint}), motivating us to learn an expert per tree.

\noindent \textbf{A Few Large Trees Dominate the Landscape.} To explore the practicality of working within Model Trees, we analyzed approximately $250k$ models from the Hugging Face model hub\footnote{We only consider models with information about their pre-training.}, for more details see \cref{app:hf_breakdown}. We find that most public models belong to a small number of large Model Trees. For instance, $20$ Model Trees already cover $50\%$ of the models. Moreover, the $196$ trees which contain $100$ or more models, collectively cover over $70\%$ of all models. \cref{fig:hf_breakdown} shows a breakdown of the top $10$ Model Trees on Hugging Face. \textit{We conclude that learning metanetworks within Model Trees is both effective and practical.}

\section{Probing Expert}
\noindent \textbf{Notation and task definition.} Consider a model $\mathcal{F}$ with $s$ layers and denote the dimension of each layer by $d_{H}$ and $d_{W}$\footnote{We reshape higher-dimensional weight tensors (e.g., convolutional layers) into 2D matrices, with the first dimension being the output channels.}. Let $X^{(1)}, \ldots, X^{(s)}$ denote the weight matrices of the layers. For brevity, we omit the layer index superscript in the notation, although in practice, we apply the described method to each layer of the model individually. In case the models use LoRA \citep{lora}, we multiply the decomposed matrices $X=BA$ and work with the full matrix. Our task is to map a weight matrix $X \in \mathbb{R}^{d_{W} \times d_{H}}$ to an output vector $\textbf{y} \in \mathbb{R}^{d_Y}$, where $\textbf{y}$ is a logits vector in classification tasks or an external semantic representation in text alignment tasks.

\subsection{Dense experts}
\label{sec:prob_form}
Building on the motivation discussed in \cref{sec:motivation}, we learn a separate expert metanetwork for each Model Tree. A simple choice for the expert architecture is a linear function. As the input is a 2D weight matrix $X \in \mathbb{R}^{d_W \times d_H}$, the linear function is a 3D tensor  $W \in \mathbb{R}^{d_H \times d_W \times d_Y }$. Formally,

\begin{equation}
\label{eq:pred}
    y_k = \sum_{ij} W_{ijk} X_{ij}
\end{equation}

Although such an expert can achieve impressive performance, its high parameter count (often exceeding 1 billion) makes it impractical due to excessive memory requirements.

\begin{figure*}[t]
\centering
\includegraphics[width=\linewidth]{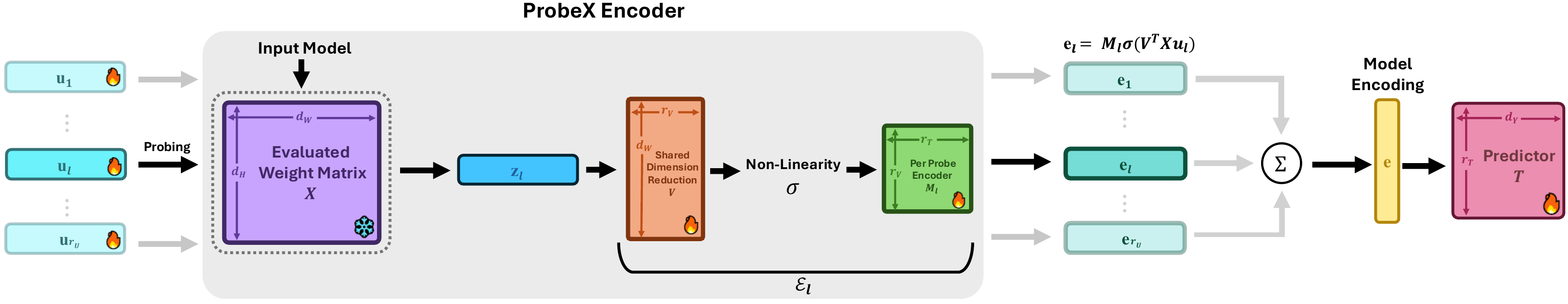}
 \caption{\textit{\textbf{ProbeX overview.}} Unlike conventional probing methods that operate only on inputs and outputs, our lightweight architecture scales weight-space learning to large models by probing hidden model layers. ProbeX begins by passing a set of learned probes, $\textbf{u}_1,\textbf{u}_2,\cdots,\textbf{u}_{r_U}$, through the input weight matrix $X$. A projection matrix $V$, shared between all probes, reduces the dimensionality of the probe responses, followed by a non-linear activation. Each probe response is then mapped to a probe encoding $\textbf{e}_l$ via a per-probe encoder matrix $M_l$. We sum the probe encodings to obtain the final model encoding $\textbf{e}$, which the predictor head maps to the task output $\textbf{y}$.}
\label{fig:probe_overview}
\end{figure*}

\subsection{Probing}
\label{sec:slpn}
Probing recently emerged as a promising approach for processing neural networks \citep{kofinas2024graph,schmidhuber,probegen}. Instead of directly processing the weights of the target model, it passes \textit{probes} (input vectors) through the model and represents the model by its outputs. As each probe provides partial information about the model, fusing information from a diverse set of probes improves representations. Passing probes through the model is typically cheaper than passing all the network weights through a metanetwork, making probing much more parameter efficient than the alternatives.

Formally, let $f_X:\mathbb{R}^{d_W} \rightarrow \mathbb{R}^{d_H}$ be the input function (e.g., a neural network). Probing methods first select a set of probes $\textbf{u}_1,\textbf{u}_2,\cdots,\textbf{u}_{r_U} \in \mathbb{R}^{d_W}$ and pass each probe $\textbf{u}_l$ through the function $f_X$, resulting in a probe response $\textbf{z}_l = f_X(\textbf{u}_l) \in \mathbb{R}^{d_H}$. A per-probe encoder $\mathcal{E}_l$ then maps the response $\textbf{z}_l$ of each probe to encoding $\textbf{e}_l \in \mathbb{R}^{d_V}$. The final model encoding $\textbf{e}$ is the sum of the encodings of all probes:

\begin{equation*}
\label{eq:probing}
    \textbf{e} = \sum_l \mathcal{E}_l(f_X(\textbf{u}_l))
\end{equation*}

Finally, a prediction head $ \mathcal{T}:\mathbb{R}^{d_V} \rightarrow \mathbb{R}^{d_Y}$, maps the model encoding to the final prediction:
\begin{equation}
\label{eq:probing_output}
    \textbf{y} = \mathcal{T}(\textbf{e})
\end{equation}

We begin with a linear probe encoder, where we can theoretically motivate our architectural design choices. We later extend our architecture to the non-linear case.

\subsection{Single layer probing experts}
\noindent \textbf{Dense vs. linear probing experts.} Traditionally, probing methods focus only on model inputs and outputs, thus avoiding many nuisance factors (e.g., neuron permutations).
However, as working within Model Trees reduces nuisance variation (see \cref{sec:motivation}), we hypothesize that probing can succeed even when applied to hidden layers. We focus on the case where $f_X(u) = X(u)$ and probing encoders are linear.

\newtheorem*{prop:generality}{\textbf{Proposition} \ref{prop:generality}}

\begin{proposition}
\label{prop:generality} Assume $\mathcal{E}_1,\ldots,\mathcal{E}_{r_U}$ are all linear operations and a sufficient number of probes. The dense expert (\cref{eq:pred}) and linear probing network (\cref{eq:probing_output}) have identical expressivity.
\end{proposition}

\begin{proof}
\cref{app:proof_1}
\end{proof}
\noindent \textbf{Deriving a single layer probing architecture.}
Having demonstrated that the linear probing framework can match the expressivity of the dense expert, we now address the primary issue of dense experts: their high parameter count. Recall that each of the $r_U$ probes has a dedicated encoder, parameterized by a large matrix. We can therefore factorize each probe encoder into a product of two matrices. The first is a dimensionality reduction matrix $V \in \mathbb{R}^{d_W \times r_V}$, \textit{shared} across probes. This matrix projects the high-dimensional outputs of $X \in \mathbb{R}^{d_W \times d_H}$ into a lower dimension $r_V$. The second matrix, $M_{l} \in \mathbb{R}^{r_V \times r_T}$ is \textit{unique} to each probe encoder and can be much smaller. By sharing the larger matrix $V$ among all probes and using a smaller, probe-specific matrix $M_{l}$, we significantly reduce the overall parameter count. Finally, the per-probe encoder is given by:
\begin{equation}
\label{eq:probex_formulation_rep}
    \mathcal{E}_l(\textbf{z}_l) = M_lV^T\textbf{z}_l
\end{equation}

The prediction head $\mathcal{T}$ is simply the matrix $T \in \mathbb{R}^{r_T \times d_Y}$. Putting everything together, our single layer probing expert, named  ProbeX (\textbf{Prob}ing \textbf{eX}pert) is:

\begin{equation}
\label{eq:probex_formulation_with_head}
    \textbf{y} = T \sum_{l} M_lV^TX^T\textbf{u}_l
\end{equation}

In \cref{prop:tucker_decomp_short} we derive our linear ProbeX (\cref{eq:probex_formulation_with_head}) from the dense expert (\cref{eq:pred}) using the Tucker tensor decomposition.

\newtheorem*{prop:tucker_decomp_short}{\textbf{Proposition} \ref{prop:tucker_decomp_short}}
\begin{proposition}
\label{prop:tucker_decomp_short}
The linear ProbeX (\cref{eq:probex_formulation_with_head}) has identical expressivity as using the dense predictor (\cref{eq:pred}), when the weight tensor $W$ obeys the Tucker decomposition:
\begin{equation*}
\label{eq:decomposition}
    W_{\text{Tucker}} = \sum_{nml} M_{nml} \cdot \textbf{t}_n  \otimes \textbf{v}_m \otimes \textbf{u}_l
\end{equation*}
\end{proposition}

\begin{proof}
\cref{app:proof_2}
\end{proof}

\noindent \textbf{Non-linear probing experts}
In \cref{prop:generality,prop:tucker_decomp_short} we establish the relation between linear ProbeX and the dense expert. To make ProbeX more expressive, we add a non-linearity $\sigma$ between the two matrices $V, M_l$ making ProbeX a factorized one hidden layer neural network:
 \begin{equation*}
     \mathcal{E}_l(\textbf{z}_l) = M_l\sigma(V^T \textbf{z}_l)
 \end{equation*}
In our experiments we chose $\sigma$ to be the ReLU function. Note the approach can easily be extended to deeper probe encoders. We present an overview of ProbeX in \cref{fig:probe_overview}. 

\begin{table*}[t]
\centering
\caption{\textbf{\textit{Training dataset class prediction results.}} In this challenging task, each model is trained on $50$ randomly selected CIFAR100 classes (out of a total of $100$). We train ProbeX tree experts to predict which of the $100$ classes were used during training. While the dense expert performs moderately well, ProbeX achieves better accuracy with roughly $\times 30$ fewer parameters.}
 \setlength{\tabcolsep}{4pt}
\begin{tabular}{lcccccccc|cc}
 & \multicolumn{2}{c}{\textbf{ResNet}}              & \multicolumn{2}{c}{\textbf{DINO}} & \multicolumn{2}{c}{\textbf{MAE}}                 & \multicolumn{2}{c}{\textbf{Sup. ViT}}          & \multicolumn{2}{c}{\textbf{MoE}} \\
 \cmidrule(lr){2-3} \cmidrule(lr){4-5} \cmidrule(lr){6-7} \cmidrule(lr){8-9} \cmidrule(lr){10-11}
           \textbf{Method}                        & \textbf{Acc. $\uparrow$} & \textbf{\# Params $\downarrow$}  & \textbf{Acc.  $\uparrow$} & \textbf{\# Params $\downarrow$} & \textbf{Acc.  $\uparrow$} & \textbf{\# Params $\downarrow$} & \textbf{Acc. $\uparrow$} & \textbf{\# Params $\downarrow$} & \textbf{Acc. $\uparrow$} & \textbf{\# Params   $\downarrow$} \\
\toprule
Random    & 0.5            &   -                      & 0.5             & -     & 0.5             & -                      & 0.5   & -                      & 0.5   & -                      \\
StatNN    & 0.631          &   -                      & 0.511          & - & 0.502           & -                      & 0.522 & -                      & 0.541 & -                      \\
Dense     & 0.713          & 105m  \paramstimes{45} & 0.614           & 59m \paramstimes{25} & 0.666 & 59m   \paramstimes{25} & 0.663 & 59m  \paramstimes{25}  & 0.664 & 282m \paramstimes{30} \\
ProbeX    & \textbf{0.842} & 2.3m                    & \textbf{0.705}           & 2.3m & \textbf{0.765}  & 2.3m  & \textbf{0.885}         & 2.3m  & \textbf{0.799}         & 9.2m        \\
\end{tabular}
\label{tab:vit_resnet_moe}
\end{table*} 

\noindent \textbf{Training ProbeX.}
For classification tasks, we use ProbeX to map model weights to logits via the cross-entropy loss (\cref{sec:exp_classification}). For representation alignment, we use a contrastive loss (\cref{sec:exp_alignment}). In all cases, we optimize $V,\textbf{u}_1,\cdots,,\textbf{u}_{r_{U}},M_1,\cdots,M_{r_{U}},T$ end-to-end. Note that while our formulation describes the case of a single layer, there is no loss of generality. Given multiple layers, we can extract an encoding from each layer using ProbeX and concatenate them. Finally, we map the concatenated encoding to the output $y$ using a matrix $T$, training everything end-to-end. Notably, training ProbeX on a single layer takes under $10$ minutes on a single small GPU (e.g., $10GB$ of VRAM).

\subsection{Handling multiple Model Trees} 
\label{sec:method_moe}
In practice, models may belong to multiple Model Trees. We therefore propose a mixture-of-tree-experts approach, consisting of a router metanetwork that maps models to their tree and a per-tree expert metanetwork. Differently from recent MoE methods  \citep{Yksel2012TwentyYO} that learn the router and experts end-to-end, we decouple the two; first learning the routing function and then the ProbeX experts. For the routing function, we opt for a fast and simple clustering algorithm. Specifically, we cluster the set of models into trees using hierarchical clustering. After completing the clustering step, we compute the center of each cluster $\hat{X}_1,\hat{X}_2,\cdots,\hat{X}_k$. The routing function assigns models to the nearest cluster in $\ell_2$:
\begin{equation}
    \label{eq:moe_routing}
    R(X) = arg\min_{k} \|X - \hat{X}_k\|_2
\end{equation}

In practice, we find that these clusters perfectly match the division into Model Trees \citep{mother,gueta2023knowledge} (see \cref{app:moe}).

\section{Model jungle dataset}
\label{sec:dataset}
We construct \textit{Model Jungle} (Model-J), a dataset that simulates the structure of real-world model repositories, with models organized into a small set of \textit{disjoint} Model Trees. These trees consist of large models that vary in architecture, task, and size, with each fine-tuned model using a set of randomly sampled hyperparameters. Model-J includes $14{,}000$ models, divided into two main splits. We shortly describe the splits here, for further details see \cref{app:dataset_details}.

\noindent \textbf{Discriminative.} We fine-tune $4{,}000$ models for image classification. These models belong to one of 4 Model Trees: i) Supervised ViT (Sup. ViT) \citep{vit}, ii) DINO \citep{dino}, iii) MAE \citep{mae}, and iv) ResNet-101 \citep{resnet}. Each model is fine-tuned (using \q{vanilla} full fine-tuning) to classify images from a random subset of $50$ out of the $100$ CIFAR100 classes.

\noindent \textbf{Generative.} We fine-tune $10,000$ Stable Diffusion (SD) \citep{stable_diffusion} personalized models \citep{dreambooth}. Each model fine-tuned on $5-10$ images, randomly sampled without replacement, originating from the same ImageNet \citep{imagenet} class. This split consists of $2$ variants each with $5,000$ models: i) 
\textit{$SD_{200}$.} A fine-grained variant with $25$ models from each class using the first $200$ ImageNet classes (mostly different animal breeds). ii) $SD_{1k}$. A low resource variant with $5$ models per class for all ImageNet classes. To save compute and storage, we follow common practice and use LoRA  \citep{lora} fine-tuning. We set aside an additional test subset of models trained on randomly selected \textit{holdout} classes, $30 \in SD_{200}$ and $150 \in SD_{1k}$.

\section{Experiments}
\label{sec:experiments}

We train ProbeX and the baselines on each layer and choose the best layer according to the validation set. For more implementation details see \cref{app:implementation_details}. We use accuracy as the evaluation metric, we also report the parameter count.

\noindent \textbf{Baselines.} Most state-of-the-art methods do not scale to large models with hundreds of millions of parameters. We therefore compare to the following baselines: i) \textit{StatNN \citep{statnn}.} This permutation-invariant baseline extracts $7$ simple statistics (mean, variance, and $5$ different quantiles) for the weights and biases of each layer. It then trains a gradient-boosted tree on the concatenated statistics. ii) \textit{Dense Expert.} Training a single linear layer on the flattened raw weights. Note that this baseline produces impractically large classifiers. E.g., a single layer classifier trained to classify $SD_{1k}$ typically has $1.4B$ parameters, twice the size of the entire SD model. We also attempted to run Neural Graphs \citep{kofinas2024graph}, but it struggled to scale to the ViT architecture even when adapted to the single-layer case (see \cref{app:implementation_details}). Since all attempts yielded near-random results, we did not report it.

\subsection{Predicting training dataset classes}
\label{sec:exp_classification}
Here, we train a metanetwork to predict the training dataset classes for models in the discriminative split of Model-J. As each model was trained on $50$ randomly selected classes out of $100$, we predict a set of $100$ binary labels, each indicating whether a specific class was included in the model’s fine-tuning data. Concretely, we train \cref{eq:probex_formulation_rep} with $100$ jointly optimized binary classification heads. This task is particularly challenging, as each class represents only $2\%$ of the model's training data, making its signature relatively weak. 

This task is quite practical; consider a model repository such as Hugging Face, which currently relies on the model metadata (e.g., model card) when searching for a model. However, these model cards are often poorly documented and lack details about the specific classes a model was trained on. In contrast, our metanetwork could allow users to filter for suitable models more effectively.

In \cref{tab:vit_resnet_moe}, we present the results of ProbeX for each Model Tree in the discriminative split. While dense expert performs better than random, ProbeX performs significantly better, improving accuracy by more than $10\%$ on average with roughly $\times 30$ fewer parameters. The MoE router (\cref{eq:moe_routing}) achieves perfect accuracy, for more details see \cref{app:moe}.

\subsection{Aligning weights to text representations}
\label{sec:exp_alignment}
We hypothesize that the weights of models conditioned on text can be aligned with a text representation. We therefore learn a mapping between the weights of models in the generative split (see \cref{sec:dataset}) and the CLIP text embeddings of the model's training dataset categories. This process creates a shared weight-text embedding space. We evaluate these aligned representations across various tasks and demonstrate strong generalization. To the best of our knowledge, ProbeX is the first method that learns weight representations with zero-shot capabilities.

\noindent \textbf{Representation alignment.} We train ProbeX to map model encodings to pre-trained text embeddings (e.g., CLIP). This mapping is supervised, as we have paired data consisting of i) model weights and ii) text embedding of the category of their fine-tuning dataset. Our training loss is similar to CLIP, i.e., the optimization objective is that the cosine similarity between the ProbeX model encoding to the ground truth class text embedding will be high, and all other classes lower. 

\begin{table}[t]
\centering
\caption{\textbf{\textit{Aligned weight-text representation results:}} 
We report the text-guided classification accuracy on both the in-distribution and holdout splits. Our method generalizes not only to unseen models trained on the same classes (in-distribution) but also to entirely new object categories in a zero-shot manner, without requiring additional training. This suggests that ProbeX successfully aligns model encodings with CLIP representations.
}
\setlength{\tabcolsep}{3pt}
\begin{tabular}{llccc}
& \textbf{Method} & \begin{tabular}[c]{@{}c@{}}\textbf{In Dist.} $\uparrow$\\ \textbf{Acc.}\end{tabular}  & \begin{tabular}[c]{@{}c@{}}\textbf{Zero-shot.} $\uparrow$\\ \textbf{Acc.}\end{tabular} & \textbf{\# Params $\downarrow$} \\

\toprule
\multirow{5}{*}{\rotatebox[origin=c]{90}{$SD_{200}$}} & Random & 0.006 & 0.033 & - \\ 
& $\text{StatNN}_{{\text{MLP}}}$ & 0.018 & 0.075 & 2.6m \\
& $\text{StatNN}_{{\text{Linear}}}$ & 0.030 & 0.147 & 689k \\
& Dense & 0.801 & 0.706 & 32m \paramstimes{13} \\
& ProbeX  & \textbf{0.973} & \textbf{0.898} & 2.5m \\
\midrule
\multirow{5}{*}{\rotatebox[origin=c]{90}{$SD_{1k}$}} & Random & 0.001 & 0.006 & - \\ 
& $\text{StatNN}_{{\text{MLP}}}$ & 0.001 & 0.029 & 2.6m \\
& $\text{StatNN}_{{\text{Linear}}}$ & 0.01 & 0.045 & 689k \\
& Dense & \textbf{0.382} & 0.343 & 210m \paramstimes{84} \\
& ProbeX  & 0.296 & \textbf{0.505} & 2.5m \\
\end{tabular}
\label{tab:zeroshot}
\end{table}

\subsubsection{Zero-shot classification}
We begin by testing the zero-shot capabilities of our aligned representation on the holdout splits of Model-J (see \cref{sec:dataset}). Specifically, given a weights-to-text mapping function, we compute the similarity between the model encoding and all possible classes. The similarity score is calculated for all held-out classes (unseen during ProbeX’s training), and the model is labeled with the class that has the highest matching score (see \cref{fig:zeroshot_overview}). We perform a similar experiment for in-distribution data (categories seen during training), i.e., a standard classification setting. In \cref{tab:zeroshot}, we show the top-1 accuracy of our method compared to the dense expert. Importantly, our method generalizes not only to unseen models trained on the same classes (i.e., in-distribution) but also to entirely new object categories (i.e., zero-shot). ProbeX detects classes unseen during training with $50\%$ accuracy when there are $150$ held-out classes and nearly $90\%$ accuracy with $30$ held-out classes. This demonstrates that ProbeX successfully aligns model encodings with CLIP’s representations.

\noindent \textbf{kNN classification.} Similarly to the zero-shot setting, using the aligned representations, kNN can correctly classify the training dataset class. The score is the average kNN distances between the text aligned ProbeX representation of the test model and the training models from this class. \cref{tab:knn_occ} compares our aligned representation with simply using raw weights, our representation performs much better.

\subsubsection{Unsupervised downstream tasks}
\label{sec:retrieval}
\textbf{Model retrieval.} Given a model, we search for the models that were trained on the most similar datasets. We use the cosine distance between the ProbeX text-aligned model representations as the similarity metric. \cref{fig:ret} shows the $3$ nearest-neighbors for $3$ query models, each fine-tuned using a different dataset. We visualize each model by showing $2$ images from its training set. Indeed, the retrieved models are closely related to the query models, showing our representation captures highly semantic attributes even in fine-grained cases. For instance, while $SD_{200}$ contains many different dog and cat breeds, our retrieval accurately returns the breed that the query model was trained on.

\begin{figure}[t]
\centering
\includegraphics[width=0.95\linewidth]{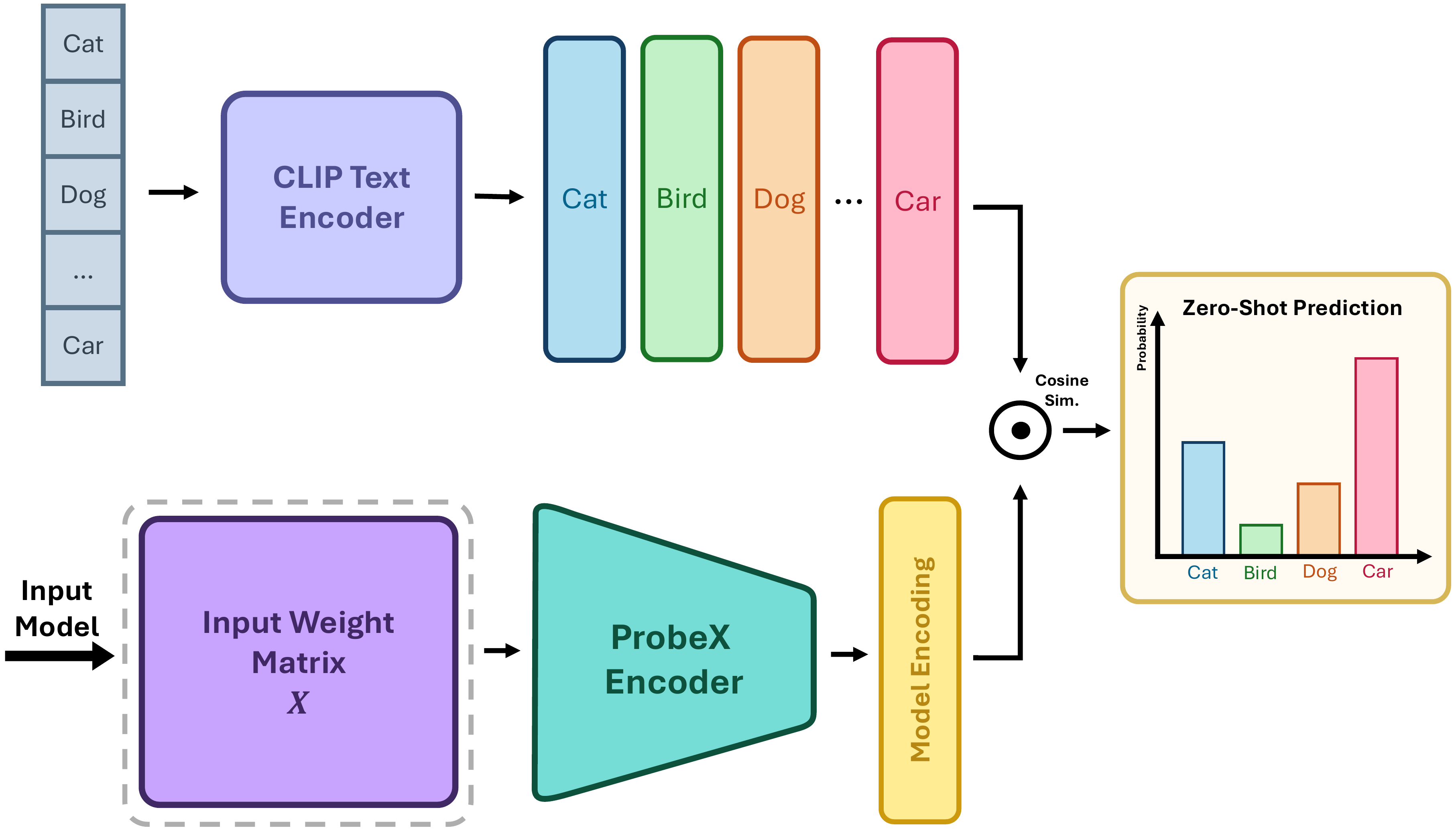}
 \caption{\textit{\textbf{Zero-shot inference overview.}} We align model weights with a pre-trained text encoder for zero-shot model classification. We extract the CLIP text embedding of each class name and use ProbeX to encode the weight matrix $X$ into a shared weight-text embedding space. Classification follows by selecting the text prompt nearest to the model weight representation $e$ using cosine similarity. This creates a CLIP-like zero-shot setting, where model weights from unseen classes are classified via text prompts.}
\label{fig:zeroshot_overview}
\end{figure}

\begin{figure*}[t]
\centering
\includegraphics[width=0.95\linewidth]{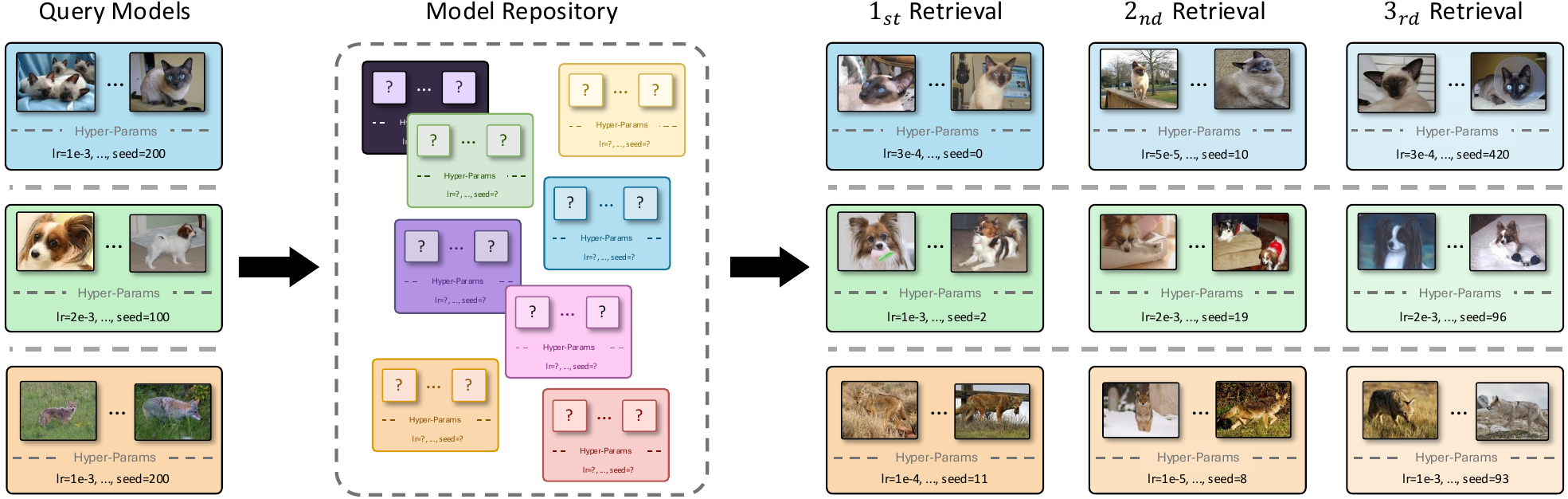}
 \caption{\textit{\textbf{Qualitative retrieval results:}} For each query model, we search for the models trained on the most similar categories, measuring similarity via the cosine distance between the text-aligned ProbeX model representations. We present the $4$ nearest neighbors for three query models, each fine-tuned on a different category. For visualization, we show two of the images used to train the model. Indeed, the retrieved models are of similar animal breeds to the query models, indicating our representations accurately capture fine-grained semantics.}
\label{fig:ret}
\end{figure*}

\noindent \textbf{One-class-classification.}
We further examine our text-aligned representations for detecting out-of-distribution models (OOD).
For each held-out class, we label it as \q{normal} and compute the average kNN distance between all test models and the training set of the normal class. Samples near the normal distribution are considered normal while others are labeled as OOD. We average the results over all classes. In \cref{tab:knn_occ} we report the mean ROC AUC score, using the kNN similarity score for separating normal and OOD models. Indeed, the results show that our method can detect OOD models much more accurately than other methods. This result remains consistent across varying numbers of neighbors, clearly demonstrating that the representation extracted by ProbeX captures more semantic relations.

\section{Ablations}
\label{sec:ablations}
\noindent \textbf{Activation function.} We ablate the need for a non-linear ProbeX using the $SD_{200}$ dataset; results are shown in \cref{tab:relu_ablation}. Interestingly, while the use of ReLU slightly improves in-distribution classification performance ($0.953$ without vs. $0.973$ with), the main benefit is in zero-shot capabilities ($0.564$ without vs. $0.898$ with). This significant difference in zero-shot performance suggests that, while the linear version of ProbeX can effectively represent the training classes, generalizing to unseen classes requires a deeper model.

\begin{table}[t]
    \centering
    \setlength{\tabcolsep}{4pt}
    \caption{\textbf{\textit{kNN and OCC results:}} Average results over all $30$ heldout classes of $SD_{200}$. ProbeX achieves the highest results for both.}

    \begin{tabular}{lccc|ccc}
    & \multicolumn{3}{c}{\textbf{kNN (Acc. $\uparrow$)}} & \multicolumn{3}{c}{\textbf{OCC (mAUC $\uparrow$)}} \\
    \cmidrule(lr){2-4} \cmidrule(lr){5-7}
    \textbf{k} & \textbf{Raw} & \textbf{Dense} & \textbf{ProbeX} & \textbf{Raw} & \textbf{Dense} & \textbf{ProbeX} \\
    \midrule
    1  & 0.833 & 0.502 & \textbf{0.913} & 0.501 & 0.561 & \textbf{0.698} \\
    2  & 0.389 & 0.525 & \textbf{0.933} & 0.502 & 0.573 & \textbf{0.702} \\
    5  & 0.417 & 0.477 & \textbf{0.872} & 0.504 & 0.610 & \textbf{0.720} \\
    All & 0.033 & 0.294 & \textbf{0.428} & 0.507 & 0.681 & \textbf{0.792} \\
    \end{tabular}
    \label{tab:knn_occ}
\end{table}

\noindent \textbf{Text encoder.} We use $SD_{200}$ to ablate the sensitivity of our method to the precise language encoder used. While CLIP performs best ($0.898$), our approach remains effective across different text encoders (e.g., $0.860$ with OpenCLIP \citep{openclip}, and $0.564$ with BLIP2 \citep{blip}), see \cref{app:ablations} for more details.

\section{Discussion}
\label{sec:discussion}

\noindent \textbf{Generalizing to unseen Model Trees.} In this paper, we focus on learning within a closed set of Model Trees. However, new Model Trees are continually added to public repositories. A primary limitation of ProbeX is its inability to generalize to these new Model Trees, requiring training new experts for new trees. Despite this drawback, ProbeX’s lightweight design and the ability to train experts independently in under $10$ minutes allows for quick integration of new experts.

\noindent \textbf{Aligning representations of other models.} When aligning model weights with text representations, we focused on SD models. In preliminary experiments, we found that aligning models from the discriminative split performs well on in-distribution classes but does not generalize to unseen (zero-shot) classes. We hypothesize that the cross-attention layers in SD models facilitate alignment between model weights and text representations. Extending this alignment to other model architectures is left for future work.

\begin{table}[t]
\centering
\caption{\textbf{\textit{Activation ablation on $SD_{200}$:}} Using ReLU slightly improves in-distribution classification, but significantly improves zero-shot classification. This suggests that while linear ProbeX represents training categories well, ReLU enhances generalization.}
    \begin{tabular}{lcc}
    
      & \textbf{In Dist. Acc. $\uparrow$}  & \textbf{Zero-shot Acc $\uparrow$}. \\
    \toprule
    No ReLU   & 0.953 & 0.564 \\
    ReLU      & 0.973 & \textbf{0.898} \\

    \end{tabular}
    \label{tab:relu_ablation}
    \vspace{-5pt}
\end{table}

\noindent \textbf{Deeper ProbeX encoders.} In this work, we used encoders with a single hidden layer. In preliminary experiments, we observed that adding more layers to the encoder reduced performance, probably due to overfitting. An intriguing direction for future research is designing deeper encoders that improve generalization or handle more complex tasks.

\section{Conclusion}
\label{sec:conclusion}
In this paper, we take the first step toward model search based solely on model weights. We identify that learning from models within the same Model Tree is significantly simpler than learning across different trees. This setting is practical as most public models belong to a few large Model Trees. We therefore introduce Probing Expert (ProbeX), a theoretically grounded architecture that scales weight-space learning to large models. As public repositories consist of multiple trees, we propose a Mixture-of-Experts approach. We demonstrate that ProbeX can embed model weights into a shared representation space alongside language embeddings, enabling text-guided zero-shot model classification.

{
    \bibliographystyle{ieeenat_fullname}
    \bibliography{main}
}

\clearpage
\appendix

\section{Proofs}
\label{app:proofs}
\subsection{Proposition 1}
\label{app:proof_1}

\begin{prop:generality}
 Assume $\mathcal{E}_1,\ldots,\mathcal{E}_{r_U}$ are all linear operations and a sufficient number of probes. The dense expert (\cref{eq:pred}) and linear probing network (\cref{eq:probing_output}) have identical expressivity.
\end{prop:generality}

\begin{proof}
We will prove both that the dense expert entails linear probing (1), and that probing entails linear experts (2).

Direction (1) is trivial, as linear probing is a composition of linear operations, it follows that the operation is a linear operation from $\mathbb{R}^{d_{W} \times d_{H}} \rightarrow \mathbb{R}^{d_Y}$. As the dense expert, parameterized as $W \in \mathbb{R}^{d_{W} \times d_{H} \times d_Y}$, can express all linear operations in $\mathbb{R}^{d_{W} \times d_{H}} \rightarrow \mathbb{R}^{d_Y}$, it clearly entails linear probing. 

Direction (2) requires us to prove that we can find a set of matrices $U,\mathcal{E}[1],\mathcal{E}[2],\cdots,\mathcal{E}[r_U],T$ such that
$\textbf{y} = T \sum_l \mathcal{E}[l]X\textbf{u}_l = \sum_{ij}W_{ijk}X_{ij}$ for every $X \in \mathbb{R}^{d_W \times d_H}$ and any $W \in \mathbb{R}^{d_{W} \times d_{H} \times d_Y}$. We show a proof by construction. Let $T = I$ (the identity matrix), $U = I$ and $\mathcal{E}[l]_{ik} = W_{ilk}$. We have:
\begin{equation}
    y_k = (T \sum_l \mathcal{E}[l]X\textbf{u}_l)_k = \sum_{ijl} W_{ilk}X_{ij}\delta_{jl}
\end{equation}
Where $\delta_{jl}$ is $1$ in the diagonal and $0$ otherwise, the $T$ is the identity matrix and cancels out. Summing over $l$, we obtain:
\begin{equation}
    y_k = \sum_{ij} W_{ijk}X_{ij}
\end{equation}
This proves that linear probing can express any dense expert.

\end{proof}

\subsection{Proposition 2}
\label{app:proof_2}

\begin{prop:tucker_decomp_short}
The linear ProbeX (\cref{eq:probex_formulation_with_head}) has identical expressivity as using the dense predictor (\cref{eq:pred}), when the weight tensor $W$ obeys the Tucker decomposition:
\begin{equation*}
    W_{\text{Tucker}} = \sum_{nml} M_{nml} \cdot \textbf{t}_n  \otimes \textbf{v}_m \otimes \textbf{u}_l
\end{equation*}
\end{prop:tucker_decomp_short}

\begin{proof}
The Tucker decomposition expresses a 3D tensor $W \in \mathbb{R}^{d_W \times d_H \times d_Y}$ by the product of a smaller tensor $M \in \mathbb{R}^{r_T \times r_V \times r_U}$ and three matrices $U \in \mathbb{R}^{d_H \times r_U}, V \in \mathbb{R}^{d_W \times r_V}, T \in \mathbb{R}^{d_Y \times r_T}$ as follows:
\begin{equation}
    W = \sum_{nml} M_{nml} \cdot \textbf{t}_n  \otimes \textbf{v}_m \otimes \textbf{u}_l
\end{equation}
Where $\otimes$ is the tensor product, and $\textbf{u}_q, \textbf{v}_q, \textbf{t}_q$ are the $q^{th}$ column vectors of matrices $U,V,T$ respectively.

The expression for the Tucker decomposition in index notation is: 

\begin{equation}
    W_{ijk} = \sum_{nml} T_{kn} M_{nml}  V_{im} U_{jl}
\end{equation}

By linearity, we can reorder the sums as:
\begin{equation}
    W_{ijk} = \sum_{n} T_{kn} \sum_{ml} M_{nml}  V_{im} U_{jl} 
\end{equation}

We can equivalently split tensor $M$ into $r$ matrices $M[1],M[2],\cdots,M[r]$, so that:
\begin{equation}
    W_{ijk} = \sum_{n} T_{kn} \sum_{ml} M[l]_{nm}   V_{im} U_{jl}
\end{equation}

Multiplying tensor $W$ by input matrix $X \in \mathbb{R}^{ d_W \times d_H}$, the result is:
\begin{equation}
    \tilde{y}_k = \sum_{ij} X_{ij} W_{ijk} = \sum_{ij} X_{ij} \sum_{n} T_{kn} \sum_{ml} M[l]_{nm}   V_{im} U_{jl}
\end{equation}

By linearity, we can reorder the sums:
\begin{equation}
    \tilde{y}_k = \sum_{n} T_{kn} \sum_{ml} M[l]_{nm} \sum_{ij} V_{im} X_{ij} U_{jl}  
\end{equation}

Rewriting $U$ using its column vectors this becomes:
\begin{equation}
    \tilde{y}_k = \sum_{n} T_{kn} \sum_{ml} M[l]_{nm} \sum_{i} V_{im} (X\textbf{u}_l)_i  
\end{equation}

Rewriting the sum over $i$ as a matrix multiplication:
\begin{equation}
    \tilde{y}_k = \sum_{n} T_{kn} \sum_{ml} M[l]_{nm} (V^TX\textbf{u}_l)_m  
\end{equation}

Rewriting the sum over $m$ as a matrix multiplication:
\begin{equation}
    \tilde{y}_k = \sum_{n} T_{kn} \sum_{l} (M[l]V^TX\textbf{u}_l)_n  
\end{equation}

Rewriting the sum over $n$ as a matrix multiplication, we finally obtain:
\begin{equation}
    \tilde{y} = T\sum_l M[l]V^TX\textbf{u}_l  
\end{equation}

\end{proof}

\clearpage
\section{Additional discussion}
\label{app:additional_discussion}
\subsection{Mechanistic vs. functional weight learning} 
\citet{schmidhuber} distinguished between two approaches to weight-space learning. The mechanistic approach treats the weights as input data and learns directly from them, while the functionalist approach (e.g., probing) interacts only with a model’s inputs and outputs. Although the functionalist approach bypasses weight-space-related nuisance factors such as permutations or Model Trees, it treats the entire model as a black box, limiting its scope. ProbeX can be seen as a blend of both approaches, enabling us to operate at the weight level while engaging with the function defined by the weight matrix. This approach may facilitate the study of different model layers’ functionalities. For instance, in the case of the MAE and Sup. ViT Model Trees, which share the same architecture, the most effective layer for our task differed between the two. This suggests that, despite having the same architecture, the two models utilize their layers for different functions. 

Similarly, for our aligned representations, the best-performing layer is a \q{query} layer in the U-Net’s encoder. However, examining the top 10 best-performing layers by in-distribution accuracy reveals that the specific \q{query} layer chosen is critical, resulting in a $6.6\%$ difference in zero-shot accuracy between the best and second-best layers. Additionally, while two of the top 10 layers are \q{out} layers and perform well on in-distribution samples, their performance drops sharply on the zero-shot task, causing a rank decrease of five places. \Cref{tab:rep_align_layer_breakdown} lists the top 10 layers by in-distribution validation accuracy alongside their zero-shot task results.

\subsection{Self-supervised learning vs. aligning representations} 
Here, we align model weights with existing representations. While weight-space self-supervised (SSL) learning \citep{schurholt2022hyper,schurholt2021self,sane} do not depend on external representations, they typically require carefully crafted augmentations and priors. Designing such augmentations for model weights is non-trivial as key nuisance factors are still being identified. We hope our work accelerates research on new weight-space SSL methods.

\subsection{Other weight-space learning tasks}
In this paper we focused on predicting the categories in a model's training dataset. However, many more weight-space learning tasks exists. As demonstrated in \cref{prop:generality}, our probing formulation is equivalent to the weight formulation, suggesting that ProbeX can potentially perform any task achievable by other mechanistic approaches. Since our focus has been on predicting the model's training dataset categories and their connection to text-based representations, extending ProbeX to these additional tasks is left for future work.

\begin{table}[t!]
\centering
\caption{\textbf{\textit{Best performing layers of $SD_{200}$:}} Rankings differ significantly between in-distribution and zero-shot tasks. Numbers in $(\cdot)$ indicate the amount the layer moved up or down in rank.}
\begin{tabular}{lcc}
\textbf{Layer Name} & \begin{tabular}[c]{@{}c@{}}\textbf{In Dist.} $\uparrow$\\ \textbf{ Acc.}\end{tabular} &\begin{tabular}[c]{@{}c@{}}\textbf{Zero-shot} $\uparrow$\\ \textbf{Acc.}\end{tabular} \\
\toprule
down.2.attentions.1.attn2.q     & 0.974 & 0.898 \rankup{0} \\
up.1.attentions.1.attn2.q       & 0.958 & 0.832 \rankdown{2} \\
up.1.attentions.0.attn2.q       & 0.952 & 0.850 \rankup{1} \\
up.1.attentions.1.attn2.out.0   & 0.946 & 0.665 \rankdown{5} \\
up.1.attentions.2.attn2.q       & 0.944 & 0.822 \rankdown{1} \\
up.1.attentions.2.attn2.out.0   & 0.920 & 0.641 \rankdown{5} \\
down.2.attentions.0.attn2.q     & 0.916 & 0.830 \rankup{2} \\
up.2.attentions.2.attn2.k       & 0.872 & 0.735 \rankup{0} \\
mid.attentions.0.attn2.q        & 0.866 & 0.848 \rankup{6} \\
up.2.attentions.1.attn2.k       & 0.830 & 0.760 \rankup{3} \\
\end{tabular}
\label{tab:rep_align_layer_breakdown}
\end{table} 

\begin{figure}[t]
    \centering
    \includegraphics[width=\linewidth]{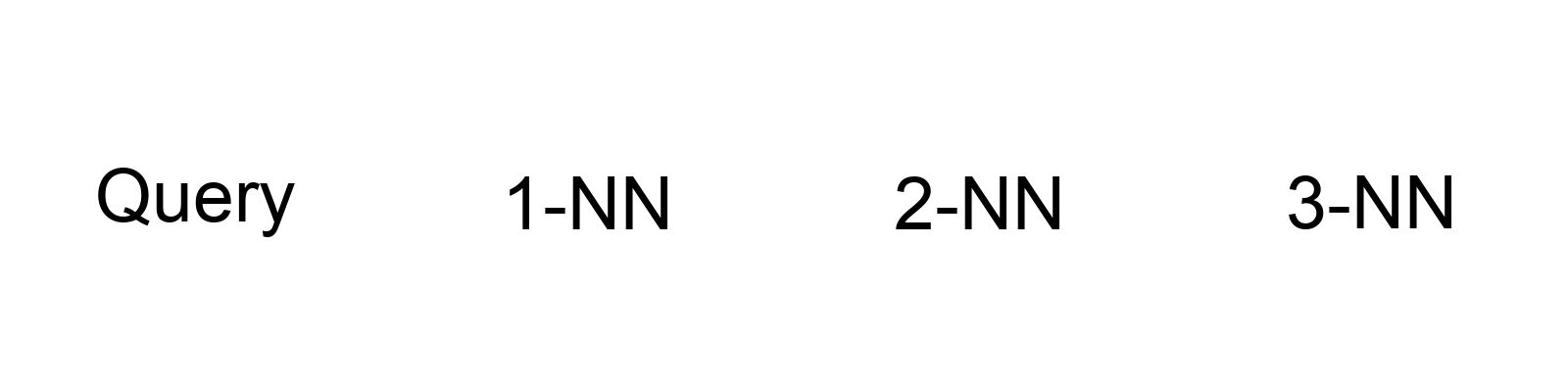} \\
    \includegraphics[width=\linewidth]{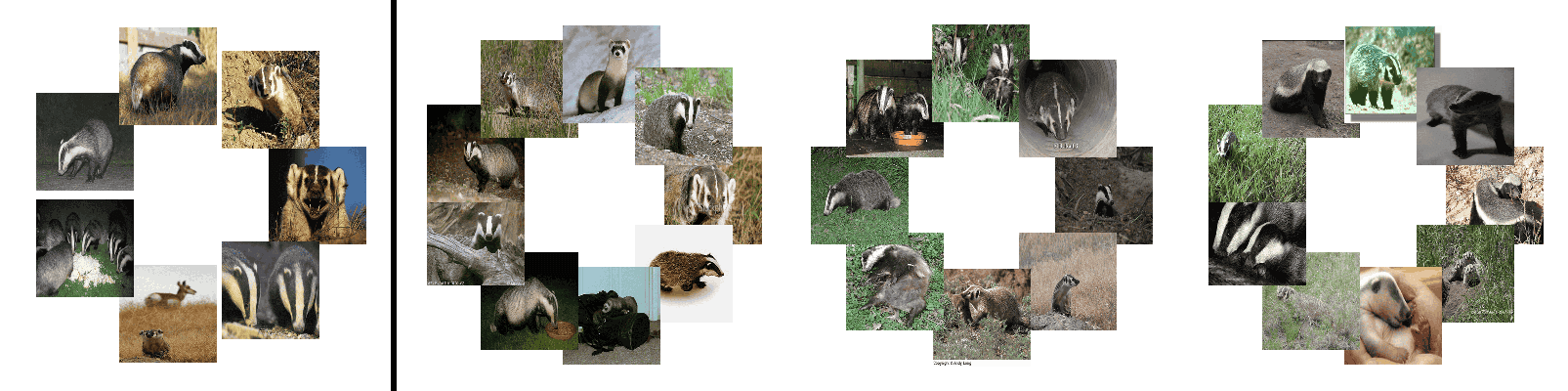} \\
    \includegraphics[width=\linewidth]{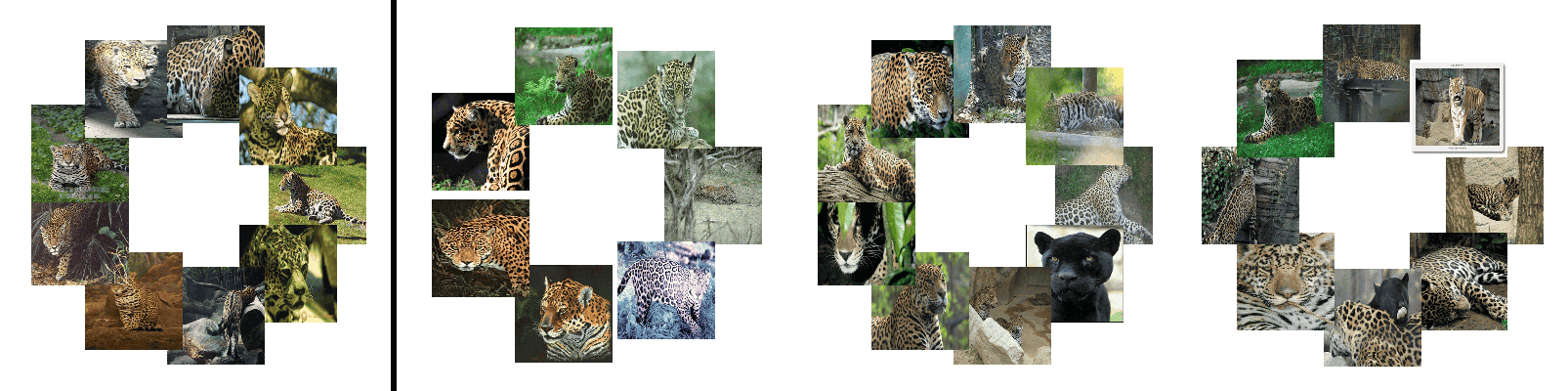} \\
    \caption{\textit{\textbf{Additional model retrieval results:}} Retrieval is performed using model weights, to visualize each model we use the set of all the images that were used to fine-tune the model.}
    \label{fig:retrieval_short}
\end{figure}

\section{Mixture-of-Experts router}
\label{app:moe}
As described in \cref{sec:method_moe}, when handling Model Graphs with multiple Model Trees, we use a mixture-of-experts approach. This involves first learning a routing function and then training a separate ProbeX model for each Model Tree.

To implement the routing function, we perform hierarchical clustering on the $\ell_2$ pairwise distances between models in the Model Graph. By calculating distances for a single model layer, this stage is significantly accelerated, enabling us to cluster Model Graphs with up to 10,000 models in under 5 minutes. Once clustering is complete, the routing function assigns each model to the nearest cluster based on $\ell_2$ distance. The number of Model Trees is determined using the dendrograms produced by hierarchical clustering. We use the \texttt{scipy} \citep{scipy} implementation with default hyperparameters.

In practice, this simple routing function achieved \textit{perfect accuracy} every time.

\section{Additional model retrieval results}
In \cref{sec:retrieval}, we presented results for the task of model retrieval. Here, we provide results for \textit{all} held-out models in $SD_{200}$. These results are not cherry-picked, and each model is visualized using the full set of images that were used for its fine-tuning. In \cref{fig:retrieval_short}, we display two additional results, in \cref{fig:retrieval_long1,fig:retrieval_long2,fig:retrieval_long3} present the remaining results.

\section{Additional ablations}
\label{app:ablations}
We provide additional ablations and expand on the ones from the manuscript.  

\begin{figure}[t]
    \centering
    \includegraphics[width=\linewidth]{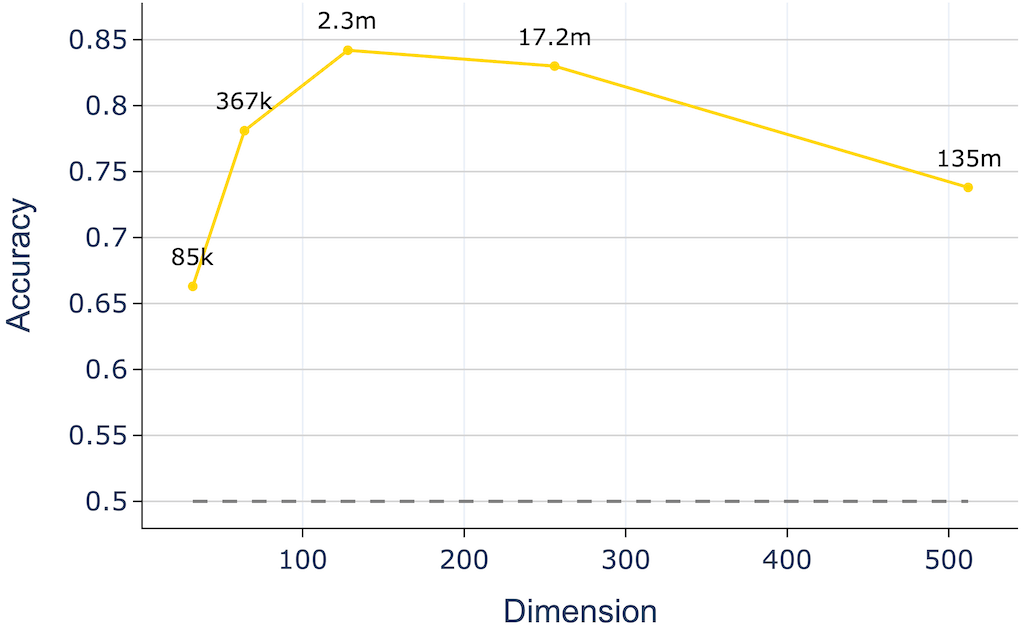}
    \caption{\textit{\textbf{$r_{(\cdot)}$ dimension ablation:}} We ablate the effect of changing the dimension of all $r_U, r_V$ and $r_T$ jointly. We can see that beyond some point the performance drops.}
    \label{fig:r_all}
\end{figure}

\begin{figure}[t]
    \centering
    \includegraphics[width=\linewidth]{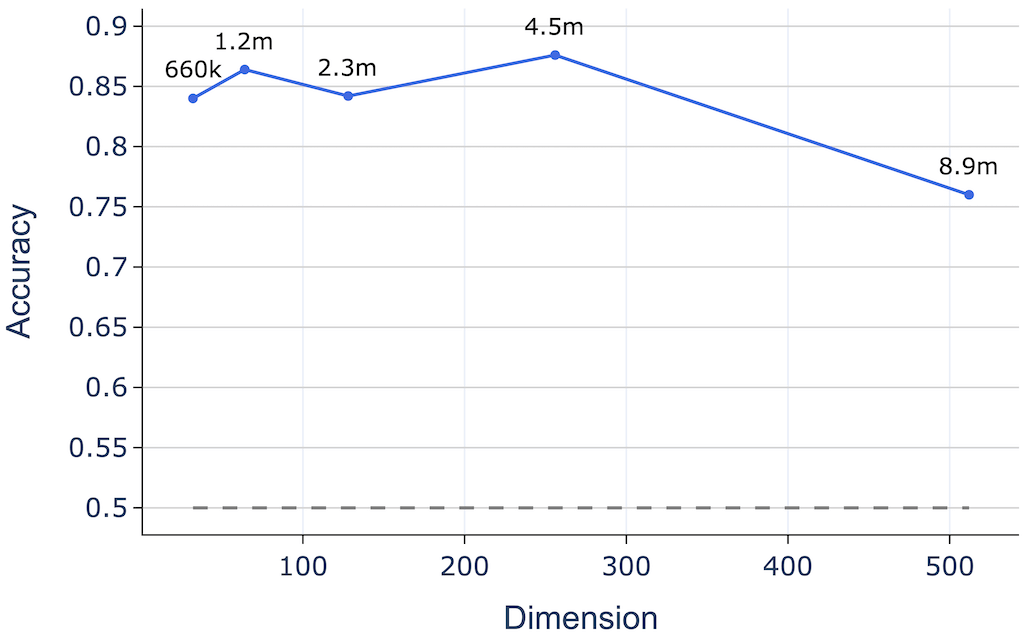}
    \caption{\textit{\textbf{Number of probes ($r_U$) ablation:}} We fix $r_V$ and $r_T$ to $128$ and change the number of probes ($r_U$). We can see that too many probes decreases the performance.}
    \label{fig:r_u}
\end{figure}

\begin{figure}[t]
    \centering
    \includegraphics[width=\linewidth]{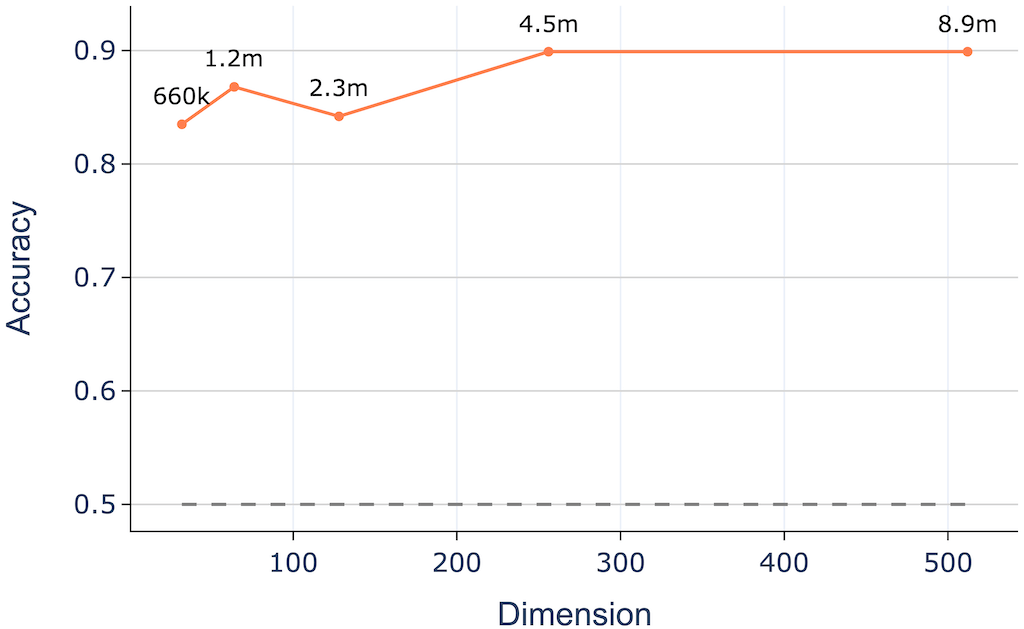}
    \caption{\textit{\textbf{Probe dimension ($r_V$) ablation:}} We fix $r_U$ and $r_T$ to $128$ and change the probe dimension ($r_V$). We can see that even a small probe dimension already results in good performance and that increasing it does not help beyond some point.}
    \label{fig:r_v}
\end{figure}

\begin{figure}[t]
    \centering
    \includegraphics[width=\linewidth]{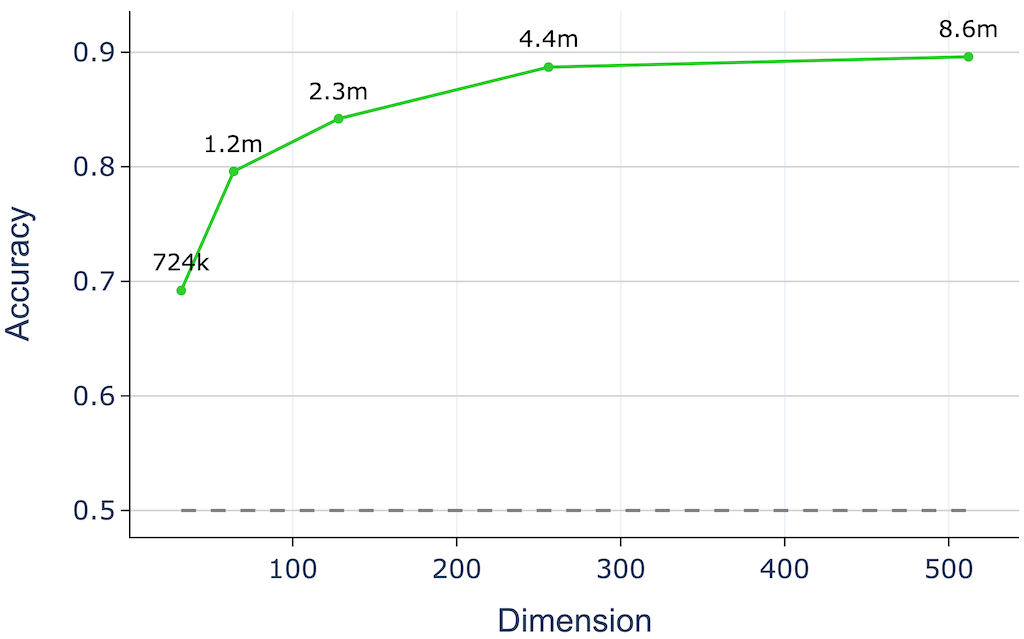}
    \caption{\textit{\textbf{Encoding dimension ($r_T$) ablation:}} We fix $r_U$ and $r_V$ to $128$ and change the encoding dimension ($r_T$). We can see that the size of the model encoding plays an important role in the performance of our method.}
    \label{fig:r_t}
\end{figure}

\subsection{$r_U,r_V,r_T$ size ablation} 
We ablate the effect of the dimensions $r_U$, $r_V$, and $r_T$ using the Sup. ViT Model Tree. We begin by examining the impact of jointly increasing all dimensions. As shown in \cref{fig:r_all}, increasing the size improves performance up to a point ($128$), after which performance begins to decline. When jointly adjusting all dimensions, the larger model size appears to be responsible for this drop. However, when we vary each dimension independently while fixing the other two at $128$, we observe a different pattern.

Starting with the number of probes ($r_U$), as shown in \cref{fig:r_u}, increasing the number of probes has minimal effect on performance until a threshold ($256$), beyond which performance drops significantly. This decline may explain the performance drop in \cref{fig:r_all}, even without an extreme increase in the parameter size.

In \cref{fig:r_v}, we observe that changing the dimension of the probes ($r_V$) has little impact on performance. Lastly, \cref{fig:r_t} shows that increasing the dimension of the encoding ($r_T$) has a dramatic effect, significantly improving performance.

\subsection{Deeper ProbeX encoders} Here, we evaluate whether deeper, non-linear ProbeX encoders outperform our single hidden-layer encoder. Specifically, we stack additional dense layers followed by non-linear activations and assess their performance. This experiment is conducted for each architecture in the Model-J dataset (i.e., ViT, ResNet, and Stable Diffusion). As shown in \cref{fig:deeper_encoders_abl_desc,fig:deeper_encoders_abl_gen}, deeper encoders tend to overfit, leading to reduced performance.

\begin{figure}[t]
    \centering
    \includegraphics[width=\linewidth]{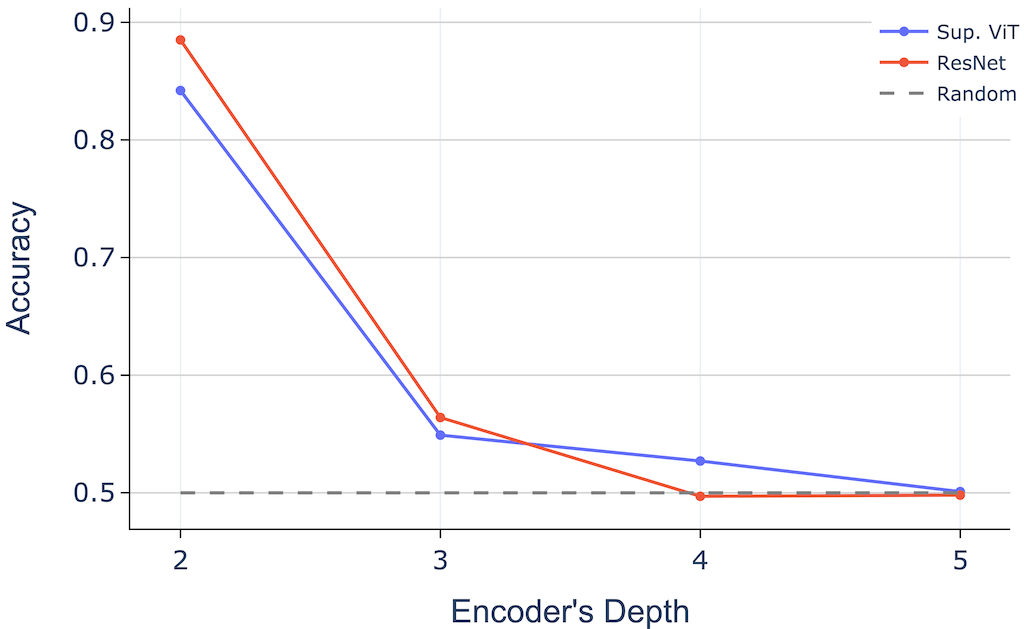}
    \caption{\textit{\textbf{Deeper ProbeX encoder ablation - discriminative}}}
    \label{fig:deeper_encoders_abl_desc}
\end{figure}

\begin{figure}[t]
    \centering
    \includegraphics[width=\linewidth]{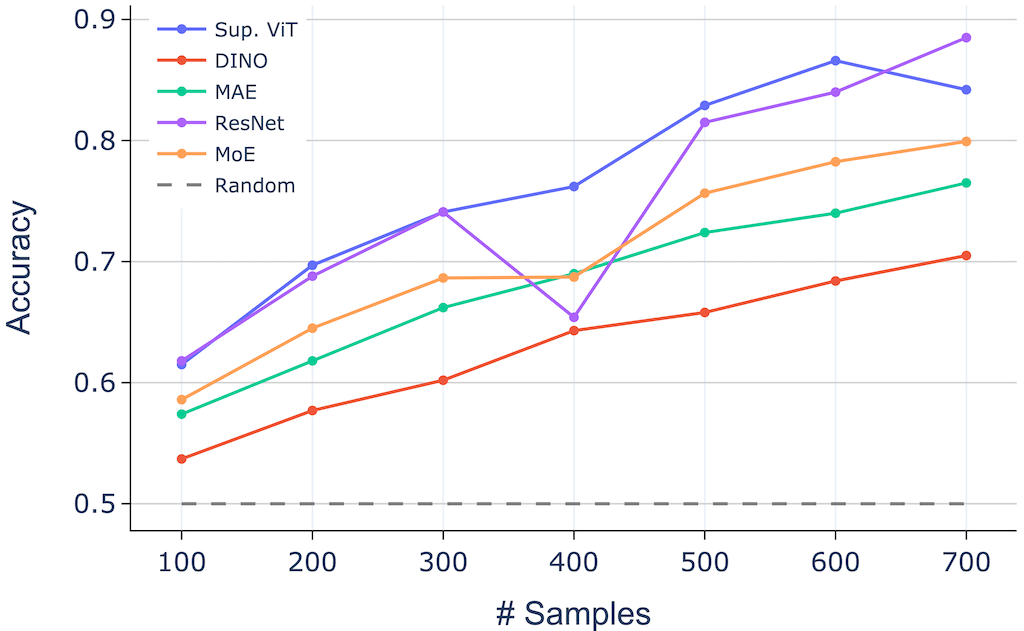}
    \caption{\textit{\textbf{Dataset size ablation}}}
    \label{fig:dataset_size_abl}
\end{figure}

\subsection{Dataset size} We examined the effect of dataset size on accuracy. Indeed, in \cref{fig:dataset_size_abl} we see that as discussed in the motivation, models that belong to the same Model Tree have positive transfer.

\begin{figure}[t]
    \centering
    \includegraphics[width=\linewidth]{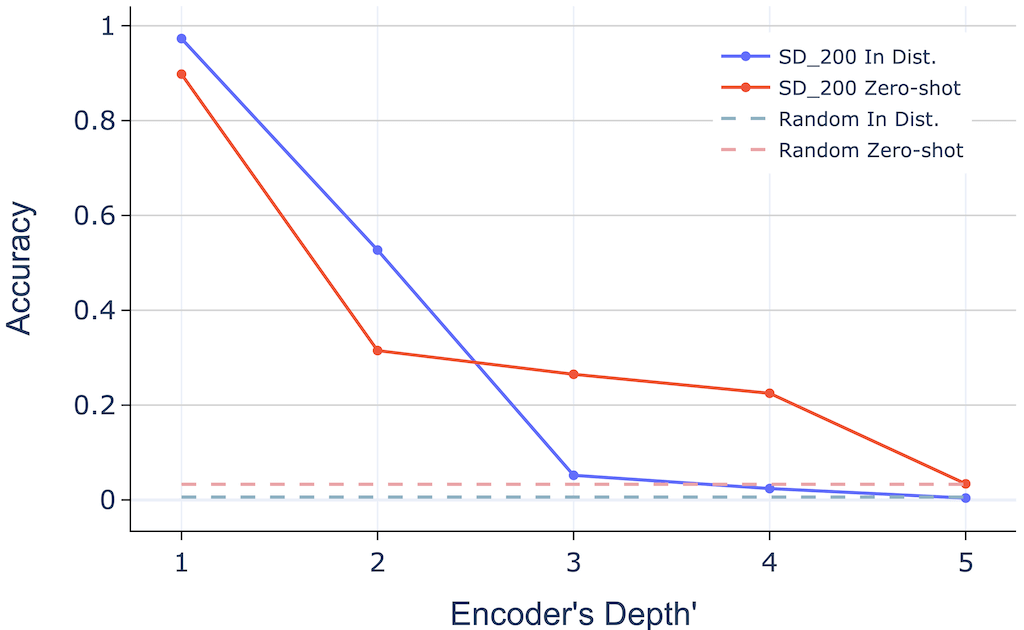}
    \caption{\textit{\textbf{Deeper ProbeX encoder ablation - $SD_{200}$}}}
    \label{fig:deeper_encoders_abl_gen}
\end{figure}

\subsection{Text encoder} We ablate whether our success in aligning model weights with CLIP rerepresentations is due to the fact Stable Diffusion was originally trained with CLIP. We perform the zero-shot experiment using $SD_{200}$ and the following text encoders. The results in \cref{tab:text_encoder_ab}, suggest that while CLIP performs best, our approach remains effective across different text encoders. This shows the robustness of ProbeX to the choice of text backbone.

\begin{table}[t]
\centering
\caption{\textbf{\textit{Text-Encoder Ablation on $SD_{200}$:}} We ablate the sensitivity of the representation alignment to different text encoders using the zero-shot experiment. While CLIP performs best, as expected due to Stable Diffusion’s training, our approach remains effective across various text encoders, demonstrating robustness to the choice of text backbone}
\begin{tabular}{lc}
\textbf{Encoder} & \textbf{Acc. $\uparrow$} \\
\midrule
   BLIP2    & 0.564 \\
  OPENCLIP  & 0.860 \\
  CLIP      & \textbf{0.898} \\

\end{tabular}
\label{tab:text_encoder_ab}
\end{table}

\section{Hugging Face Model Graph analysis}
\label{app:hf_breakdown}

Our presented statistics regarding Model Trees are based on the \q{hub-stats} Hugging Face dataset\footnote{\url{https://huggingface.co/datasets/cfahlgren1/hub-stats}}. This dataset, maintained by Hugging Face, is automatically updated daily with statistics about Hugging Face models, datasets, and more. We used a version from late September 2024, when there were \q{only} about $800{,}000$ models hosted on Hugging Face. We utilized the \texttt{base\_model} property from model cards and aggregated based on it. However, since not all models on Hugging Face use this property, these statistics are not $100\%$ accurate and may contain some bias. Additionally, \cref{fig:hf_growth} also uses the \q{hub-stats} dataset and is based on the graphs shown at \url{https://huggingface.co/spaces/cfahlgren1/hub-stats}.

\begin{table}[t!]
\centering \caption{\textbf{\textit{Hyperparameters overview - Discriminative split}}}
    \resizebox{\linewidth}{!}{%
        \begin{tabular}{l@{\hskip5pt}c@{\hskip5pt}}    
            \textbf{Name} & \textbf{Value} \\
            \midrule
            \texttt{lr} & \begin{tabular}[c]{@{}c@{}}$[5e-4, 3e-4, 1e-4, 9e-5,$ \\$7e-5, 5e-5, 3e-5]$\end{tabular}\\
            \texttt{lr\_scheduler} & \begin{tabular}[c]{@{}c@{}}linear, cosine, cosine-with-restarts,\\ constant, constant-with-warmup\end{tabular}\\
            \texttt{weight\_decay} & \begin{tabular}[c]{@{}c@{}}$[5e-2,3e-2,1e-2,9e-3,$\\$7e-3,5e-3]$\end{tabular}\\ 
            \texttt{epochs} & $[2-9]$ \\
            \texttt{random\_crop} & T,F\\
            \texttt{random\_flip} & T,F\\
            \texttt{batch\_size} & $64$ \\
            \texttt{fine-tuning type} & Full Fine-tuning\\
        \end{tabular}
    }
\label{tab:desc_hparams_overview}
\end{table}

\section{Dataset details}
\label{app:dataset_details}
Existing weight-space learning datasets and model zoos \citep{schurholt2022model,honegger2023sparsified} primarily consist of models that are randomly initialized. This means that each model in such datasets serves as the root of a distinct Model Tree containing only that model. As demonstrated in \cref{sec:motivation}, learning from such Model Graphs is significantly more challenging, highlighting the need for our approach of learning within Model Trees. Furthermore, existing datasets primarily consist of small models, typically with only thousands of parameters per model. \textit{As such, we cannot utilize the existing and established weight-space learning datasets.} 

To address this, we introduce the Model Jungle dataset (Model-J), which simulates the structure of public model repositories. Each of our fine-tuned models is trained using a set of hyperparameters sampled uniformly at random. Discriminative models share the same set of possible hyperparameters, summarized in \cref{tab:desc_hparams_overview}. Generative models, in contrast, use a different set of hyperparameters detailed in \cref{tab:gen_hparams_overview}. Notably, in the generative split, our Model Trees have multiple levels of hierarchy, as models were fine-tuned from SD1.2 to SD1.5. This structure is designed to simulate public model repositories, where Model Trees often exhibit multiple levels of hierarchy. In \cref{fig:supervised_hist,fig:dino_hist,fig:mae_hist,fig:resnet_101_hist} we provide a summary of the test accuracy the models in the discriminative split converged to. In \cref{fig:supervised_acc_vs_lr,fig:dino_acc_vs_lr,fig:mae_acc_vs_lr,fig:resnet_101_acc_vs_lr} we plot these accuracies as a function of the model's learning rate.

For the discriminative split, we use the following models as the Model Tree roots taken from \textit{Hugging Face}: 
\begin{itemize}
    \item \url{https://huggingface.co/google/vit-base-patch16-224}
    \item \url{https://huggingface.co/facebook/vit-mae-base}
    \item \url{https://huggingface.co/facebook/dino-vitb16}
    \item \url{https://huggingface.co/microsoft/resnet-101}
\end{itemize}

\begin{table}[t!]
\centering \caption{\textbf{\textit{Hyperparameters overview - generative Split}}}
    \resizebox{\linewidth}{!}{%
        \begin{tabular}{l@{\hskip5pt}c@{\hskip5pt}}    
            \textbf{Name} & \textbf{Value} \\
            \midrule
            \texttt{\# images} & [5-10]\\ 
            \texttt{lr} & \begin{tabular}[c]{@{}c@{}}$[5e-4, 3e-4, 1e-4, 9e-5,$ \\$7e-5, 5e-5, 3e-5]$\end{tabular}\\
            \texttt{prompt\_template} &   \begin{tabular}[c]{@{}c@{}}a photo of a, a picture of a,\\ a photograph of a, an image of a,\\ cropped photo of the,  a rendering of a\end{tabular}\\
            \texttt{base\_model} & [SD1.2, SD1.3, SD1.4, SD1.5] \\
            \texttt{steps} & $[450,500,550,600,650,700]$ \\
            \texttt{fine-tuning type} & LoRA\\
            \texttt{random\_crop} & T,F\\
            \texttt{rank} & 16 \\
            \texttt{batch\_size} & $1$ \\
            
        \end{tabular}
    }
\label{tab:gen_hparams_overview}
\end{table}

\begin{table}[t!]
\centering
\caption{\textbf{\textit{Model Jungle dataset summary.}} We train over $14{,}000$ models, covering different architectures, tasks and model sizes. Each model uses randomly sampled hyper parameters}
\begin{tabular}{l@{\hskip6pt}l@{\hskip5pt}l@{\hskip5pt}c@{\hskip5pt}c@{\hskip5pt}}
\textbf{Name} & \textbf{FT Type} & \textbf{Task} & \textbf{Size} & \textbf{\# Classes} \\
\toprule
DINO & Full FT &\begin{tabular}[l]{@{}l@{}} Att. classification\end{tabular}  & 1000 & 50/100  \\                                          
MAE & Full FT &\begin{tabular}[l]{@{}l@{}} Att. classification\end{tabular}  & 1000 & 50/100  \\                                          
Sup. ViT & Full FT & Att. classification & 1000 & 50/100 \\         
ResNet   & Full FT &Att. classification & 1000 & 50/100    \\ 
$SD_{200}$ & LoRA & Fine-grained & 5000 & 200              \\ 
$SD_{1k}$  & LoRA & Few shot & 5000 & 1000             \\ 
\end{tabular}
\label{tab:dataset_overview}
\end{table} 

\section{Implementation details}
\label{app:implementation_details}

\begin{figure}[t]
    \centering
    \includegraphics[width=\linewidth]{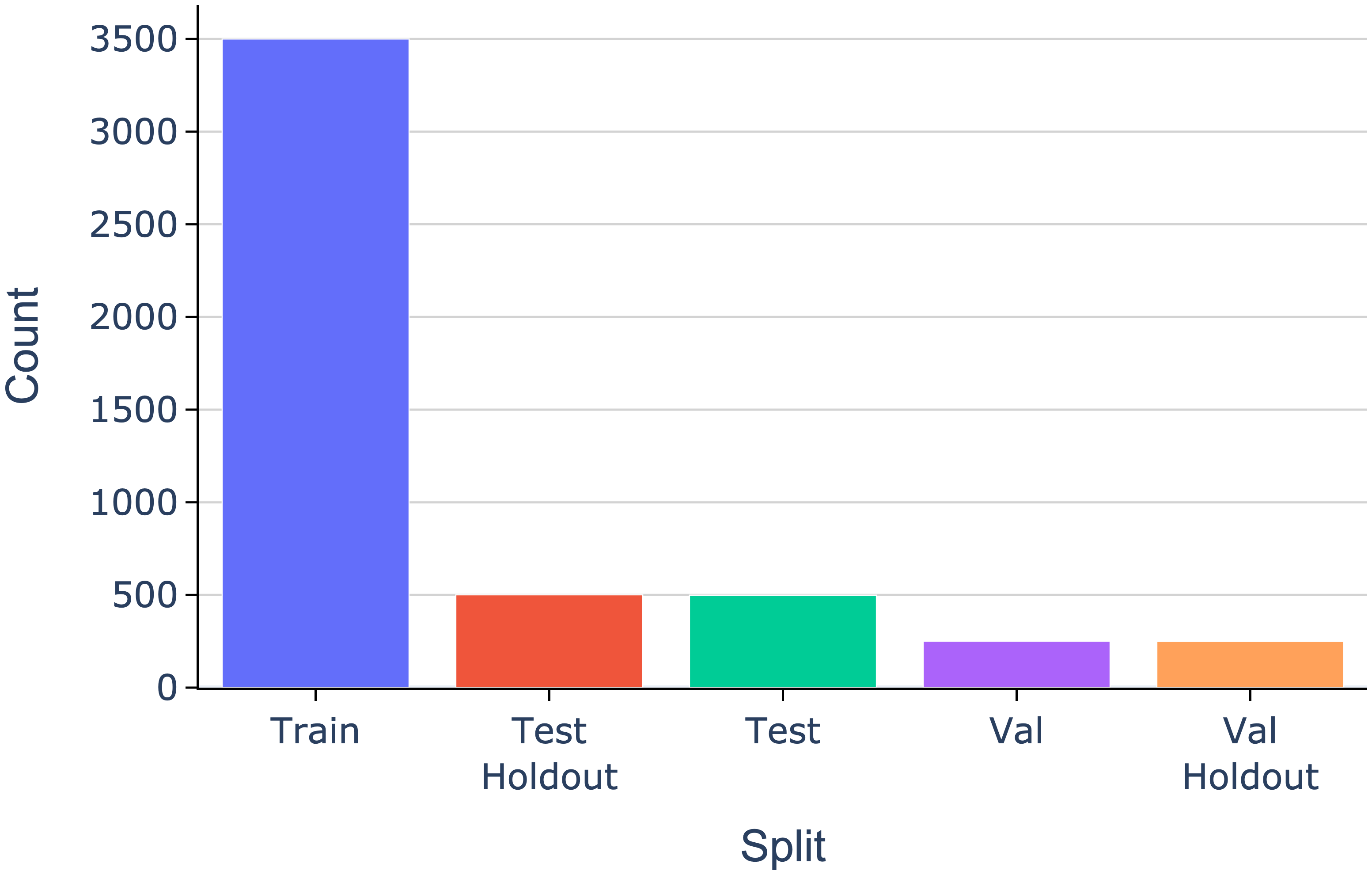}
    \caption{\textit{\textbf{Distribution of splits in the generative split}}}
    \label{fig:sd_1k_split}
\end{figure}

\begin{figure*}[t]
    \centering
    \begin{minipage}[t]{0.48\linewidth}
        \centering
        \includegraphics[width=\linewidth]{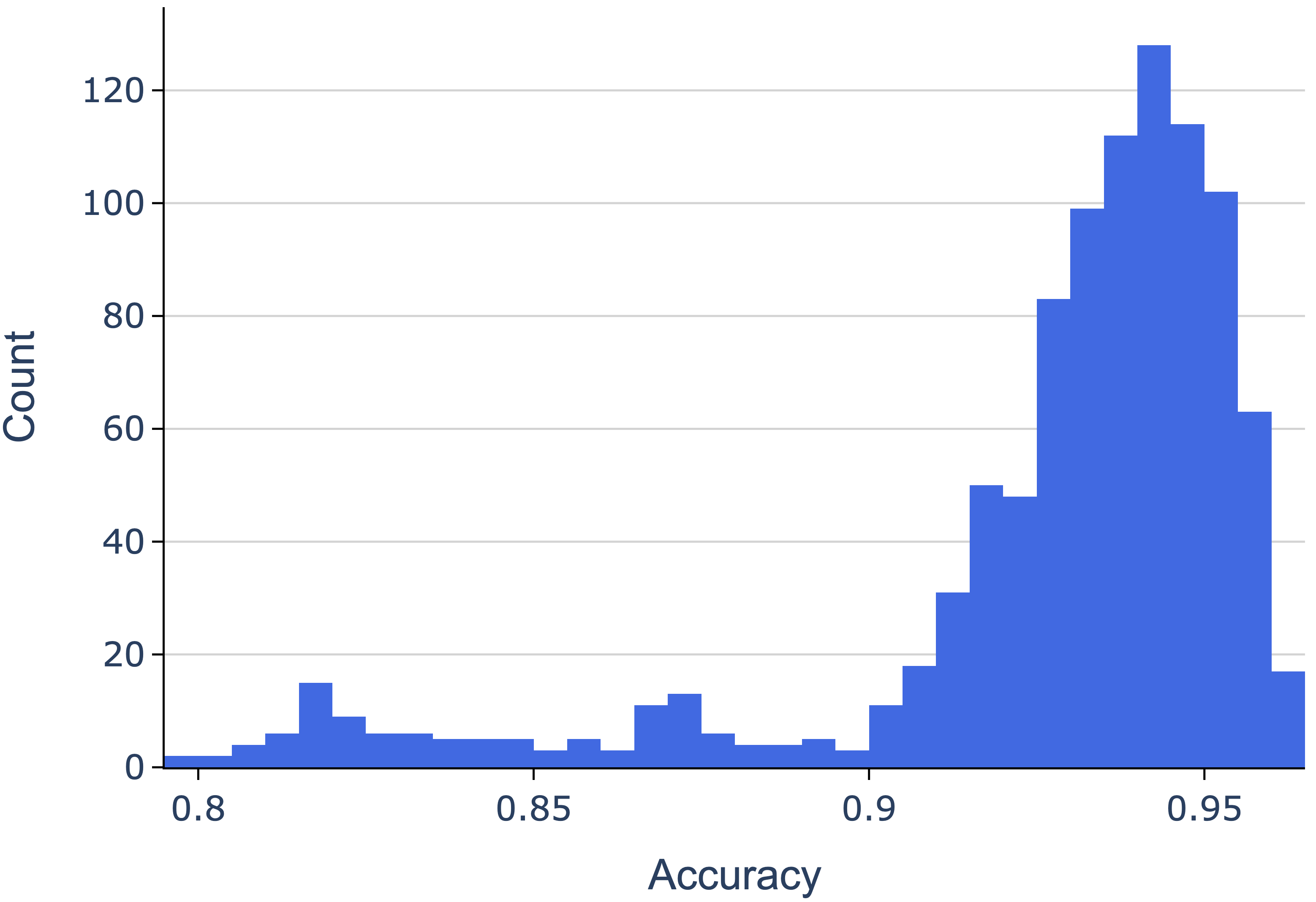}
        \caption{\textbf{\textit{Sup. ViT - Test accuracy distribution}}}
        \label{fig:supervised_hist}
    \end{minipage}
    \hfill
    \begin{minipage}[t]{0.48\linewidth}
        \centering
        \includegraphics[width=\linewidth]{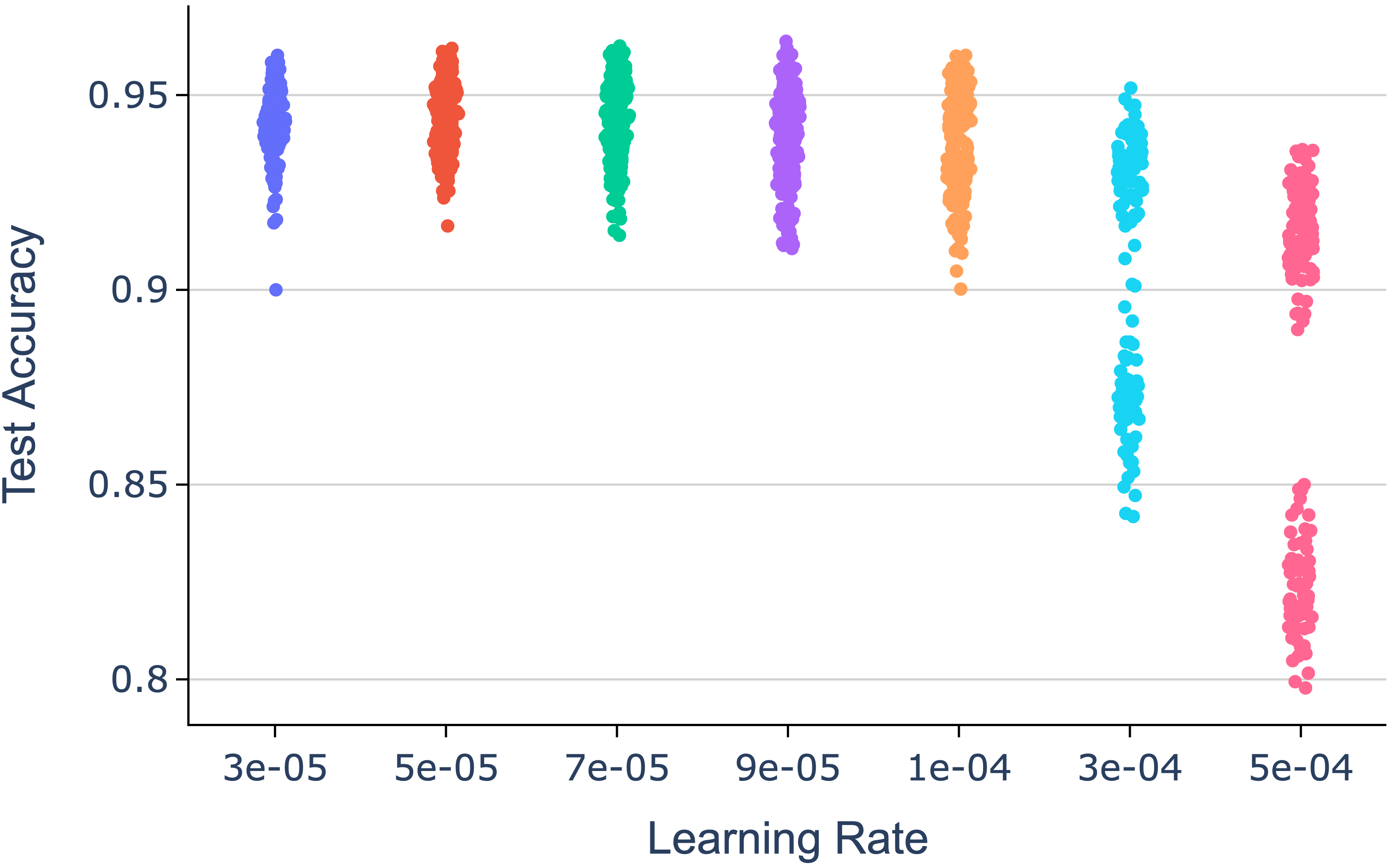}
        \caption{\textbf{\textit{Sup. ViT - Effect of learning rate on test accuracy}}}
        \label{fig:supervised_acc_vs_lr}
    \end{minipage}
    \label{fig:supervised_figures}
\end{figure*}

\subsection{Experimental setup} We use the Model-J dataset presented in \cref{sec:dataset}, split into $70/10/20$ for training, validation, and testing. Given the significant variation in results between layers, we train ProbeX for $500$ epochs on each layer and select the best layer and epoch based on the validation set. We use the Adam optimizer with a weight decay of $1e{-}5$ and a learning rate of $1e{-}3$. The number of probesm probe dimensions, and encoder dimension are set to $r_U=r_V=r_T=128$.

\subsection{Baselines}  
\subsubsection{Neural Graphs baseline}  
As mentioned in the \cref{sec:experiments}, we attempted to use \citep{kofinas2024graph} as a baseline but were unable to scale the method to models in our dataset. Here, we provide additional details about this attempt. Neural Graphs \citep{kofinas2024graph} is a graph-based approach that treats each bias in the network as a node and each weight as an edge. These methods scale quadratically with the number of neurons, leading to computational challenges when applied to larger models. To address this, we adapted the Neural Graphs approach to the single-layer case, but even in this simplified scenario, it required a relatively low hidden dimension to run on a $24$GB GPU. Since this baseline yielded near-random results in our experiments with discriminative models, we chose not to include its results in the tables.  

\subsubsection{StatNN}  
For the discriminative split, we used StatNN as a baseline by training XGBoost on the StatNN features. To compare StatNN in the case of text-aligned representations, we replaced XGBoost with an MLP, allowing the baseline to be trained with the same contrastive objective used for ProbeX. We developed two StatNN variants: i) $StatNN_{Linear}$: A single linear layer trained on top of the StatNN features. ii) $StatNN_{MLP}$: A deeper architecture designed to match the parameter count of our method.

\clearpage

\begin{figure}[t]
    \centering
    \includegraphics[width=\linewidth]{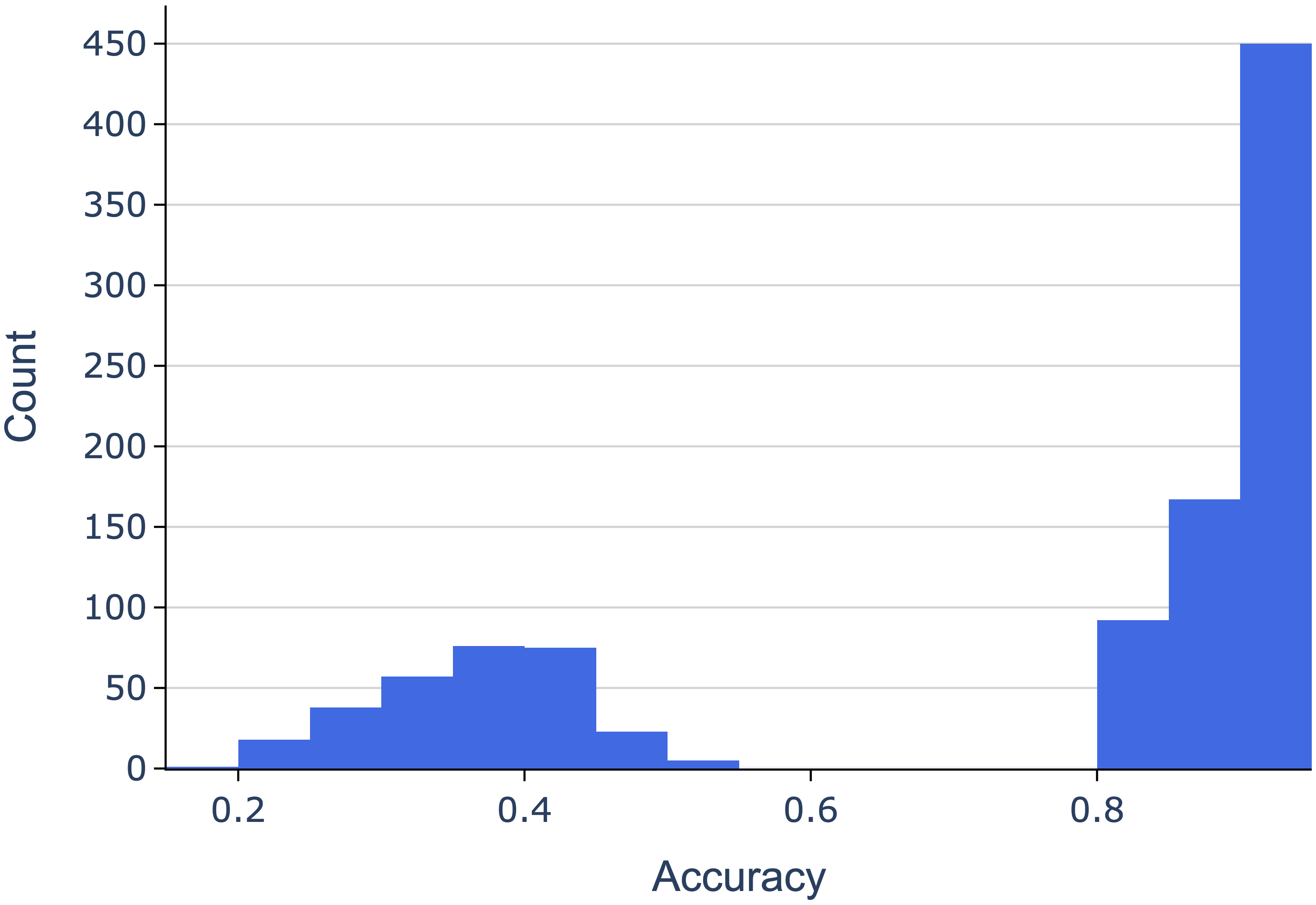}
    \caption{\textbf{\textit{DINO - Test accuracy distribution}}}
    \label{fig:dino_hist}
\end{figure}

\begin{figure}[t!]
    \centering
    \includegraphics[width=\linewidth]{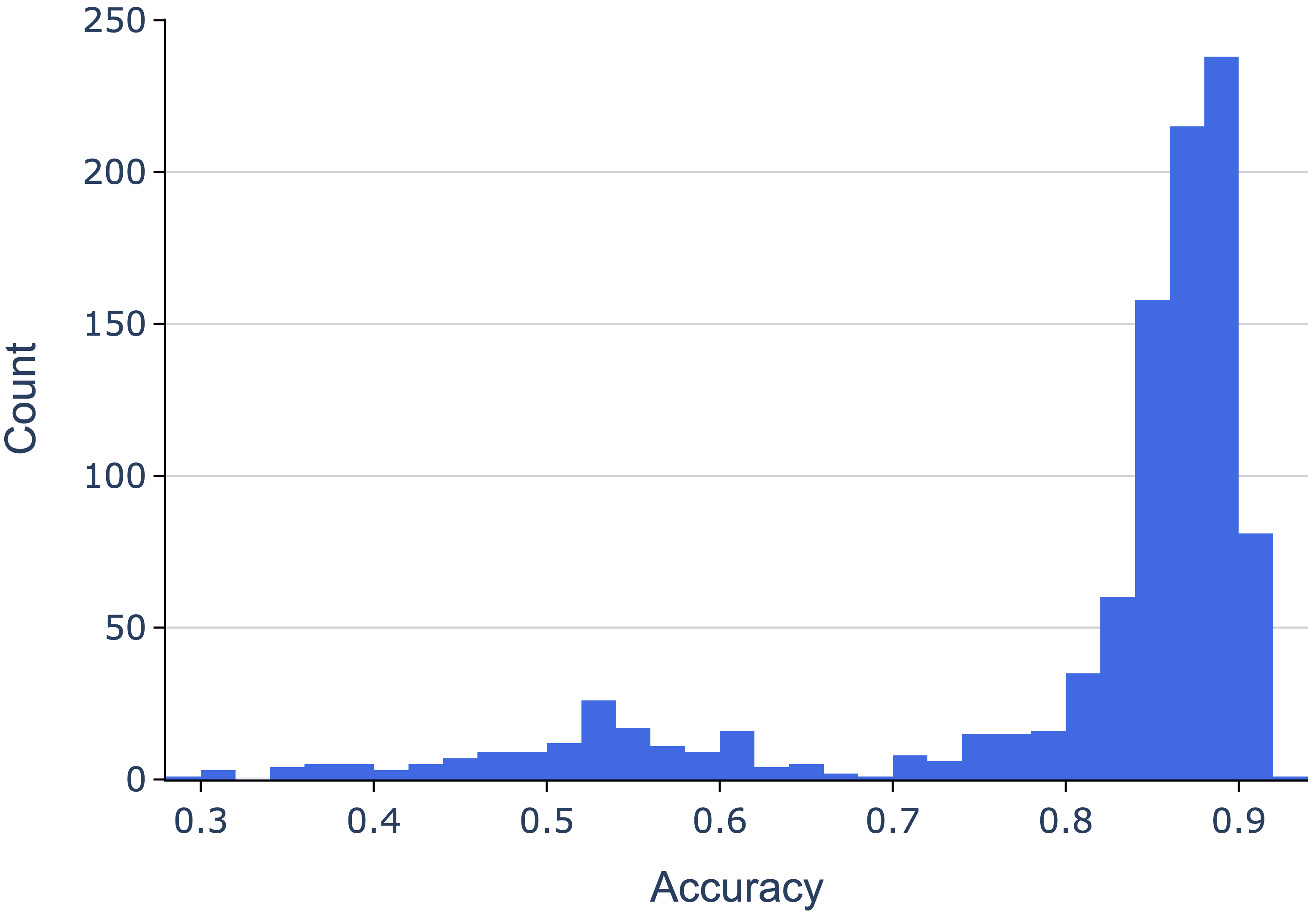}
    \caption{\textbf{\textit{MAE - Test accuracy distribution}}}
    \label{fig:mae_hist}
\end{figure}

\begin{figure}[t!]
    \centering
    \includegraphics[width=\linewidth]{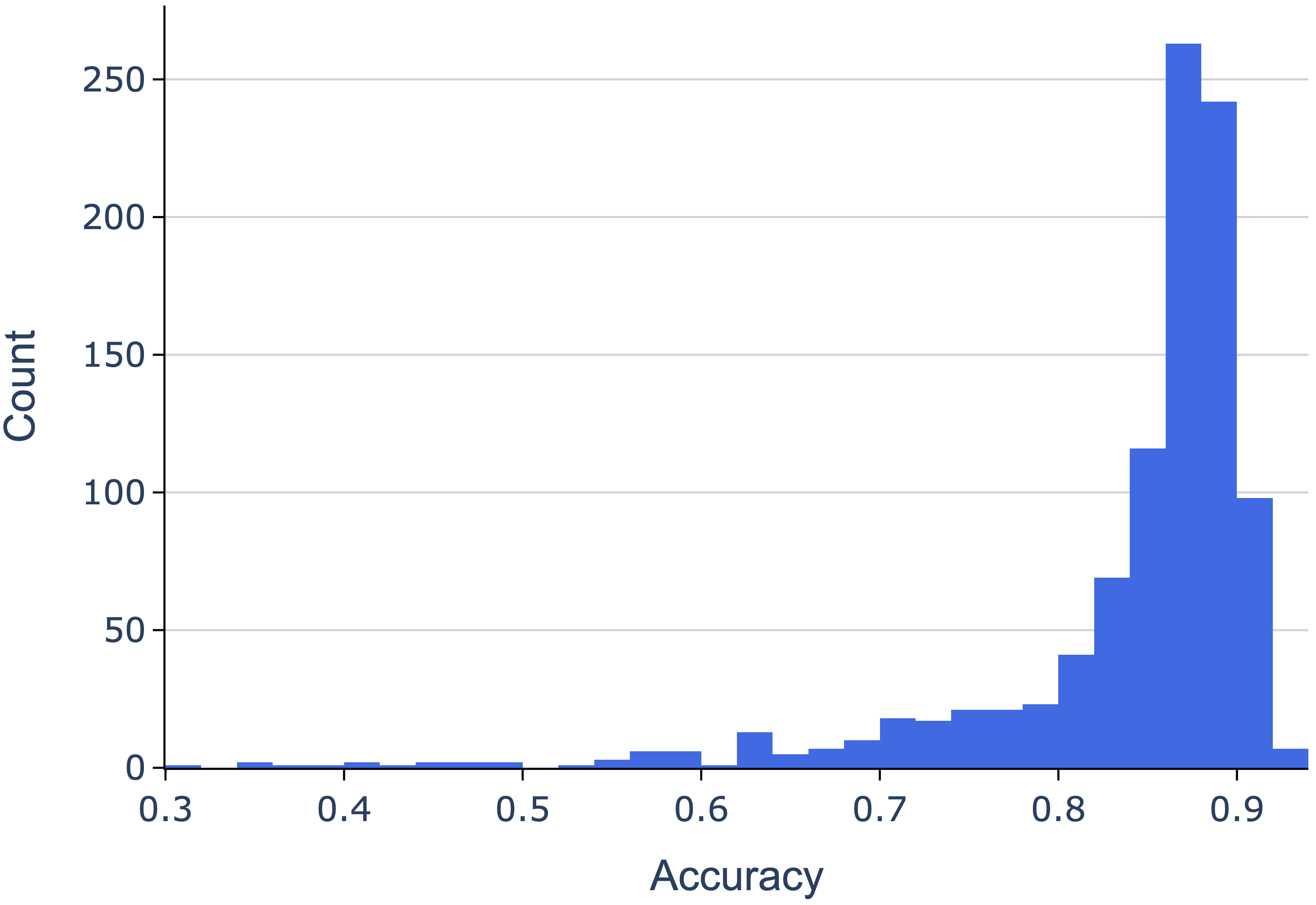}
    \caption{\textbf{\textit{ResNet - Test accuracy distribution}}}
    \label{fig:resnet_101_hist}
\end{figure}

\begin{figure}[t]
    \centering
    \includegraphics[width=\linewidth]{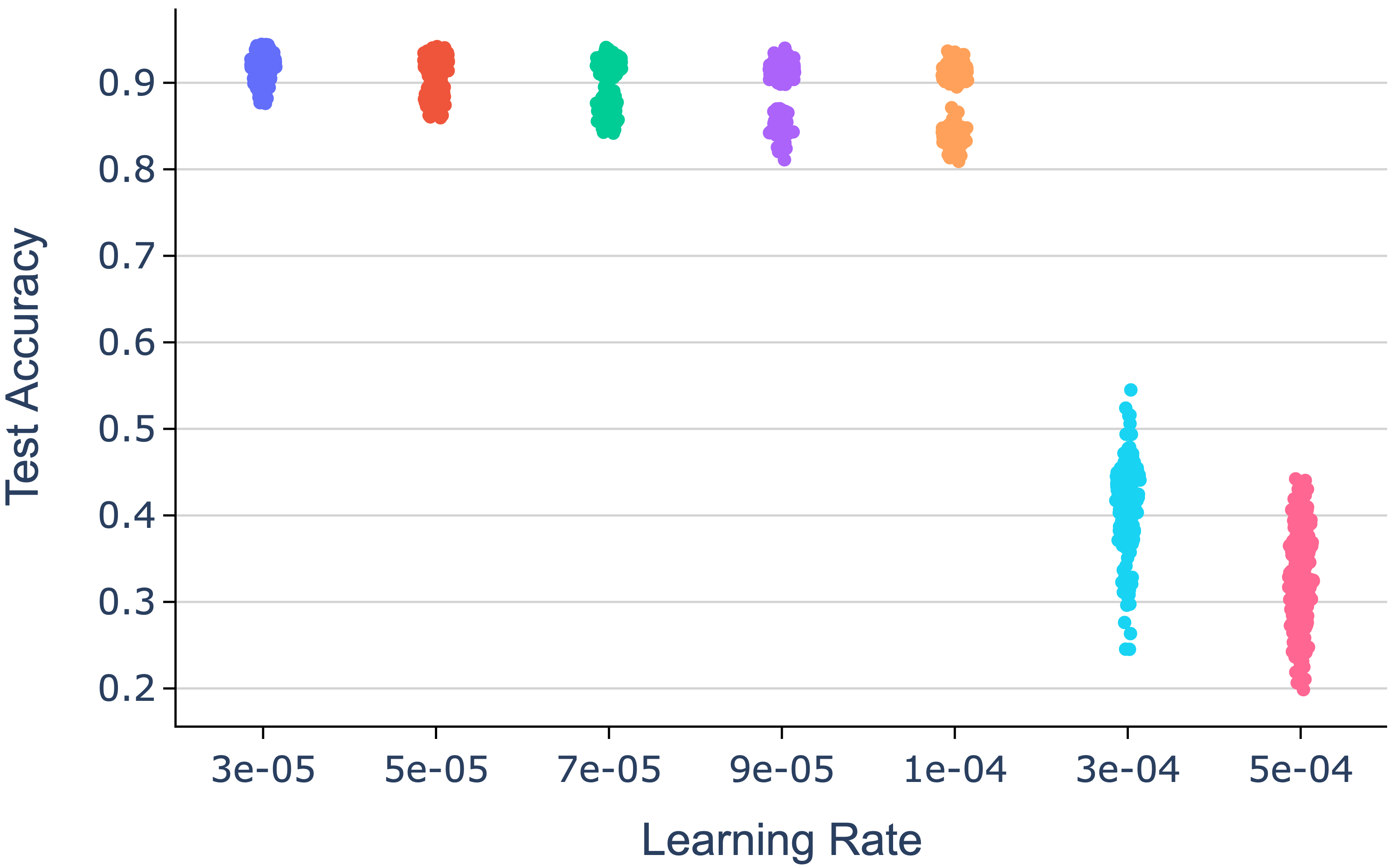}
    \caption{\textbf{\textit{DINO - Effect of learning rate on test accuracy}}}
    \label{fig:dino_acc_vs_lr}
\end{figure}

\begin{figure}[t!]
    \centering
    \includegraphics[width=\linewidth]{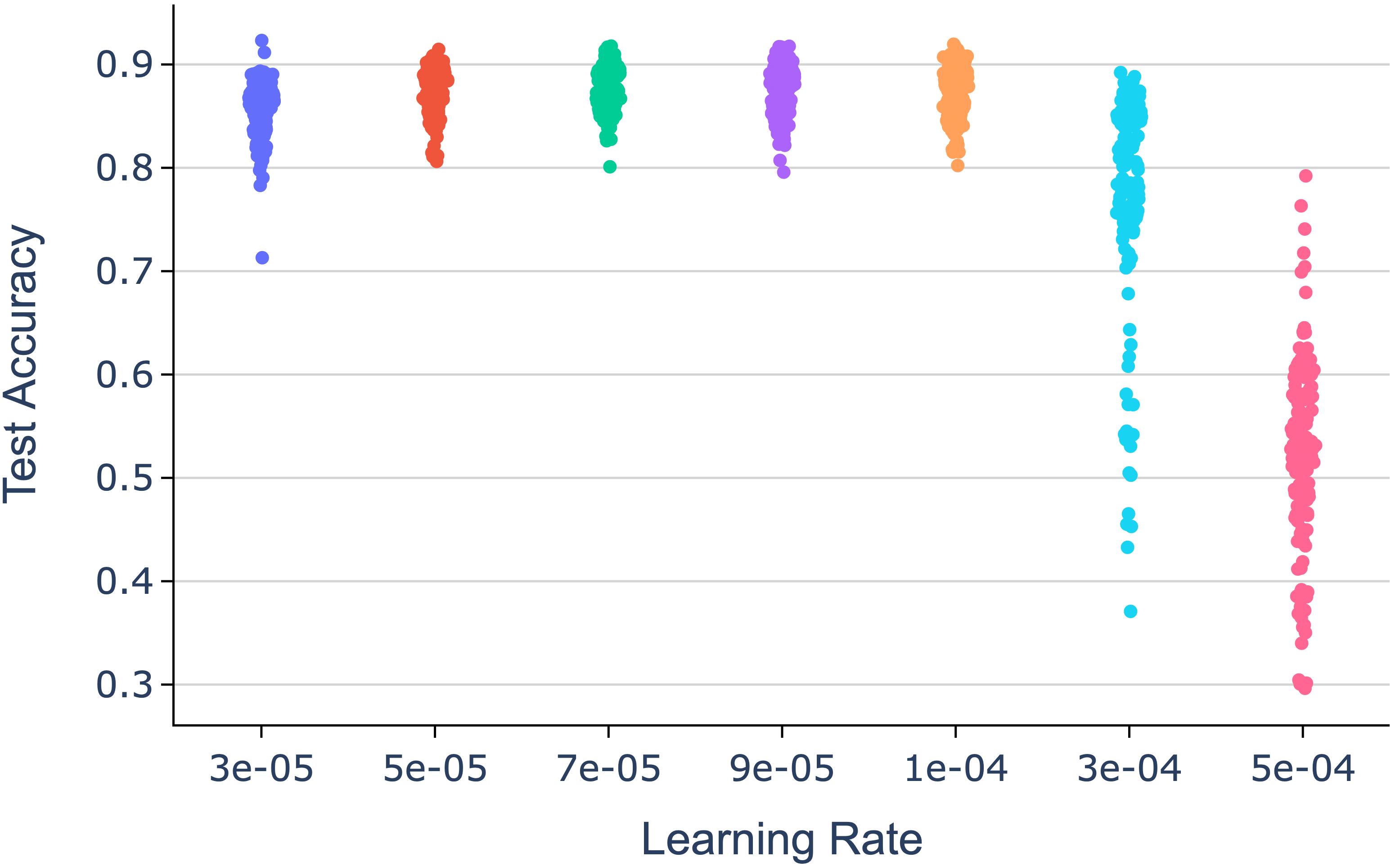}
    \caption{\textbf{\textit{MAE - Effect of learning rate on test accuracy}}}
    \label{fig:mae_acc_vs_lr}
\end{figure}

\begin{figure}[t!]
    \centering
    \includegraphics[width=\linewidth]{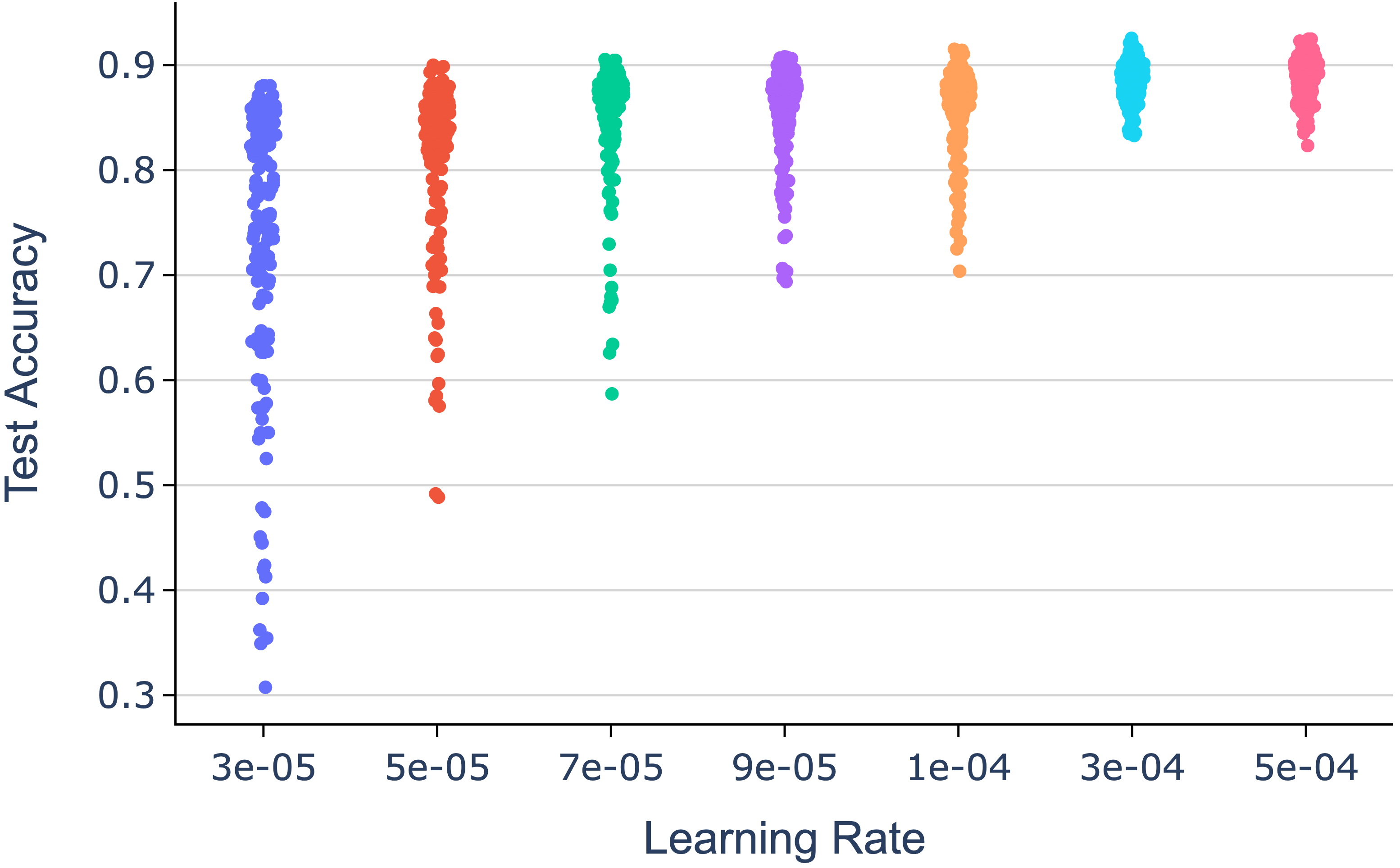}
    \caption{\textbf{\textit{ResNet - Effect of learning rate on test accuracy}}}
    \label{fig:resnet_101_acc_vs_lr}
\end{figure}

\begin{figure}[t]
    \centering
    \includegraphics[width=0.88\linewidth]{figs/appendix/retrieval/title_row.png} \\ 
    \includegraphics[width=0.88\linewidth]{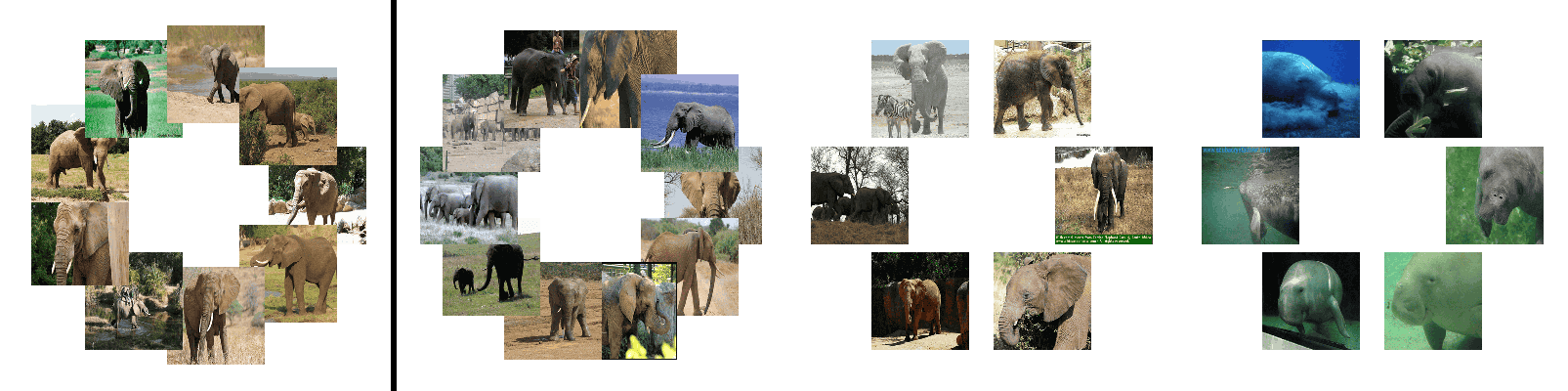} \\ 
    \includegraphics[width=0.88\linewidth]{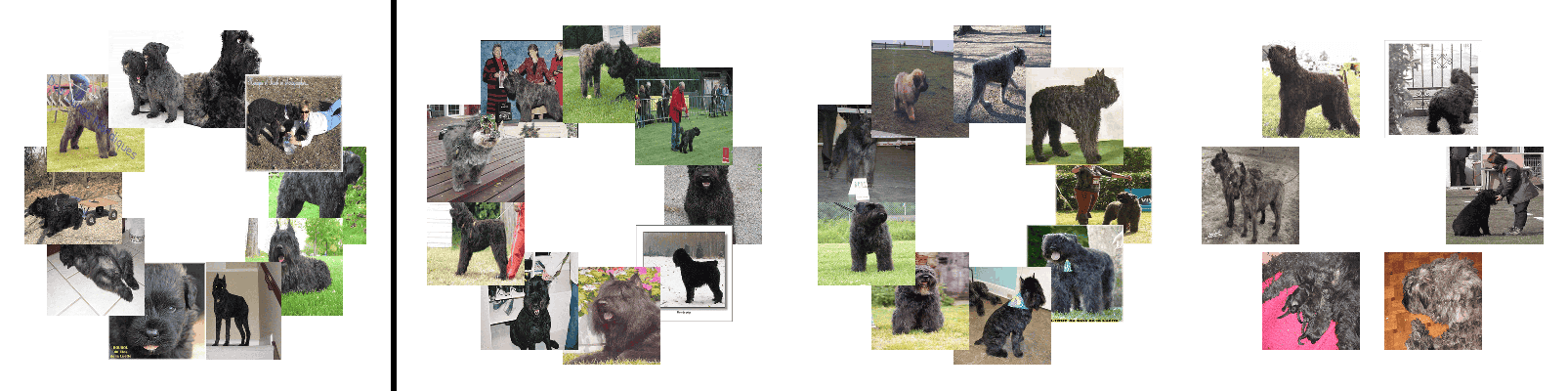} \\ 
    \includegraphics[width=0.88\linewidth]{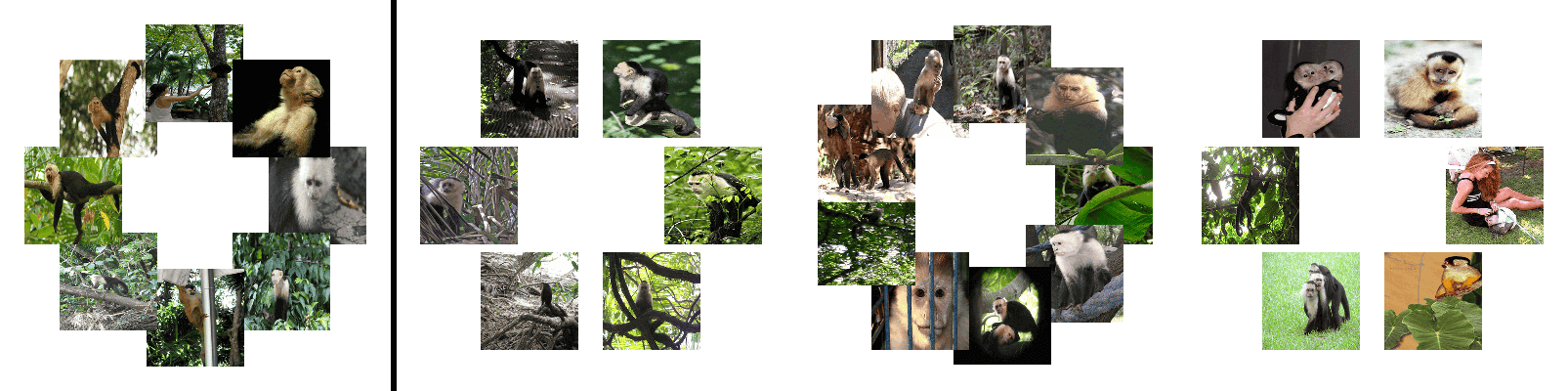} \\ 
    \includegraphics[width=0.88\linewidth]{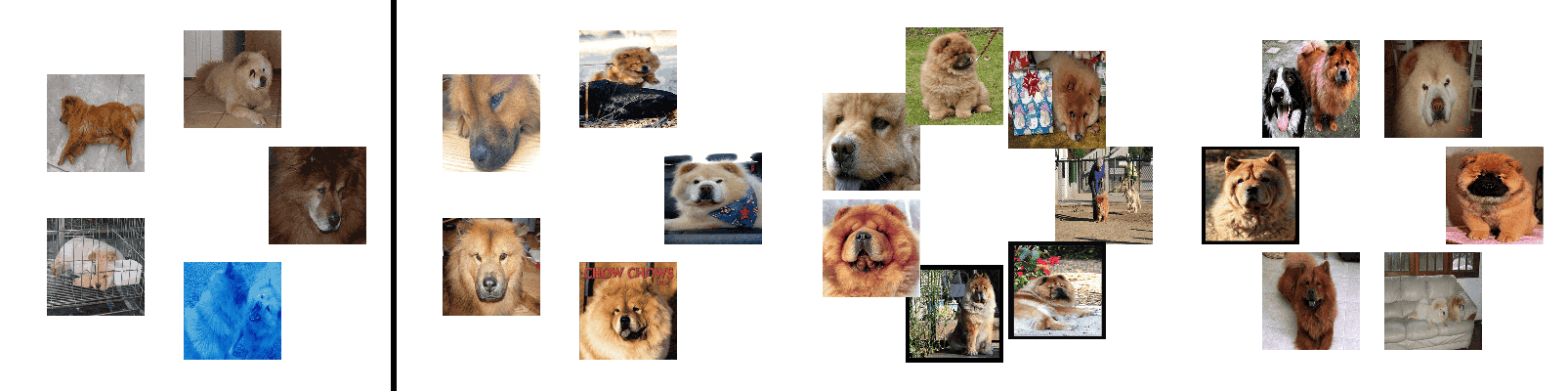} \\ 
    \includegraphics[width=0.88\linewidth]{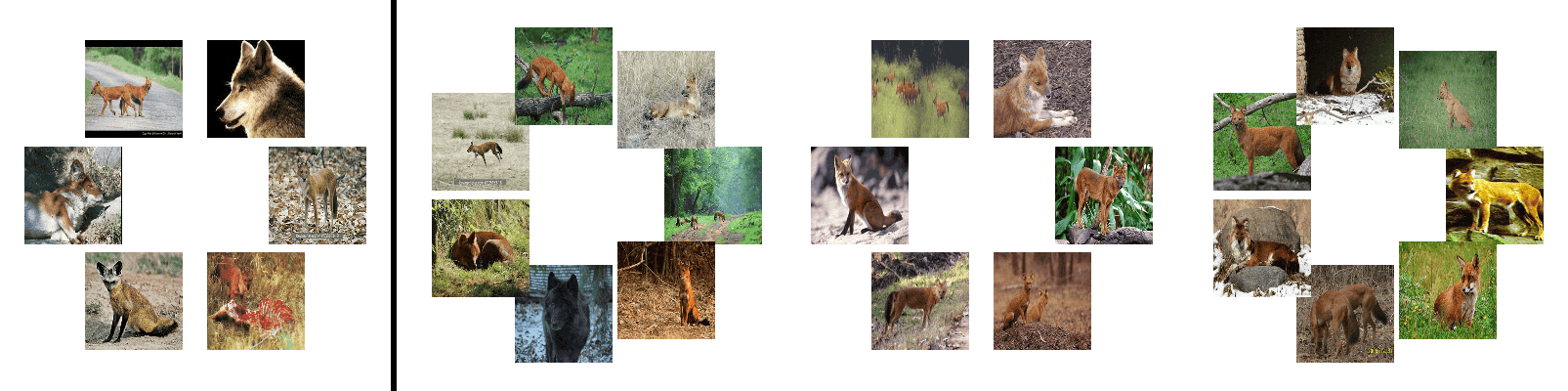} \\ 
    \includegraphics[width=0.88\linewidth]{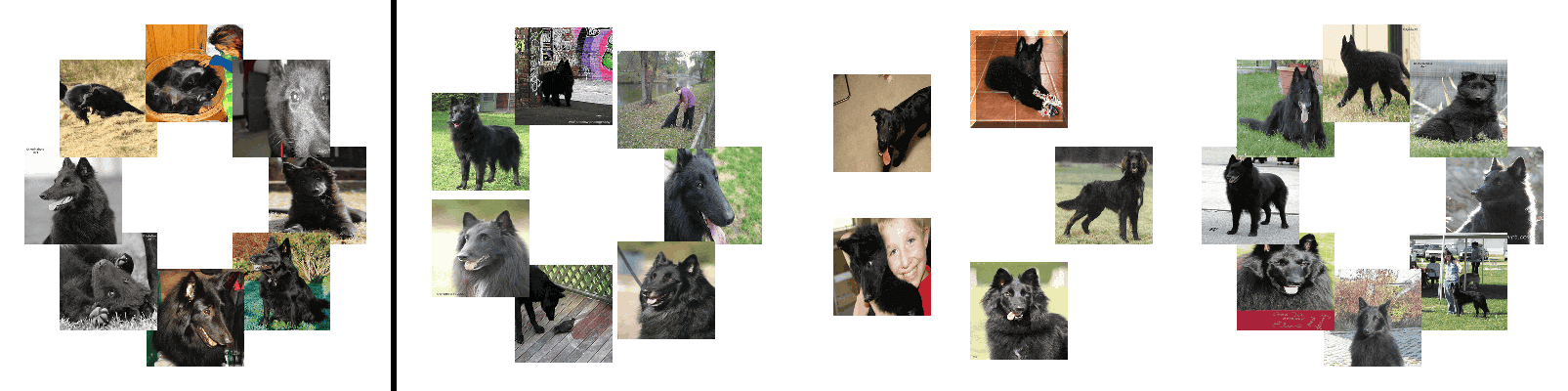} \\ 
    \includegraphics[width=0.88\linewidth]{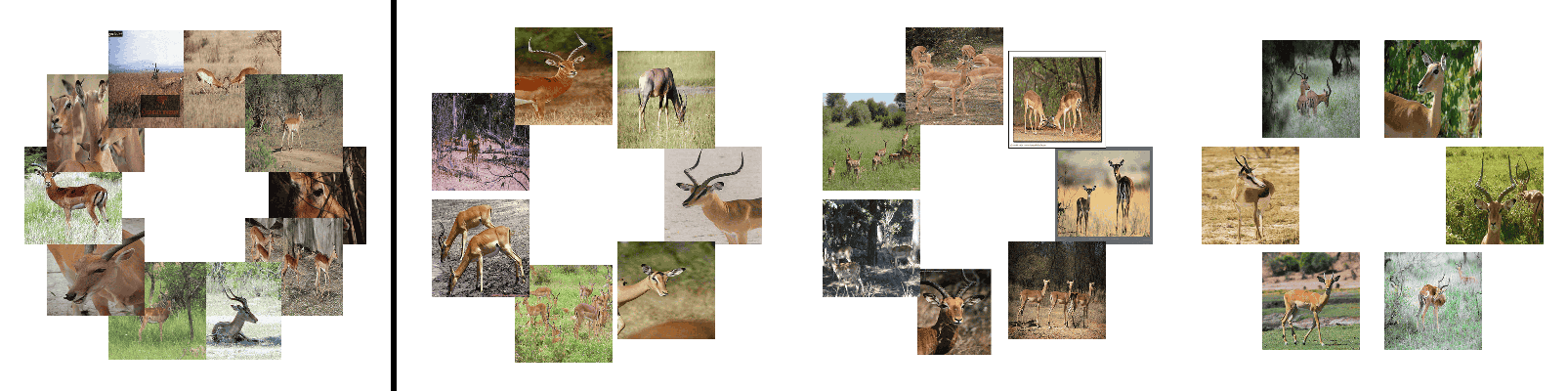} \\ 
    \includegraphics[width=0.88\linewidth]{figs/appendix/retrieval/jaguar_row.png} \\ 
    \includegraphics[width=0.88\linewidth]{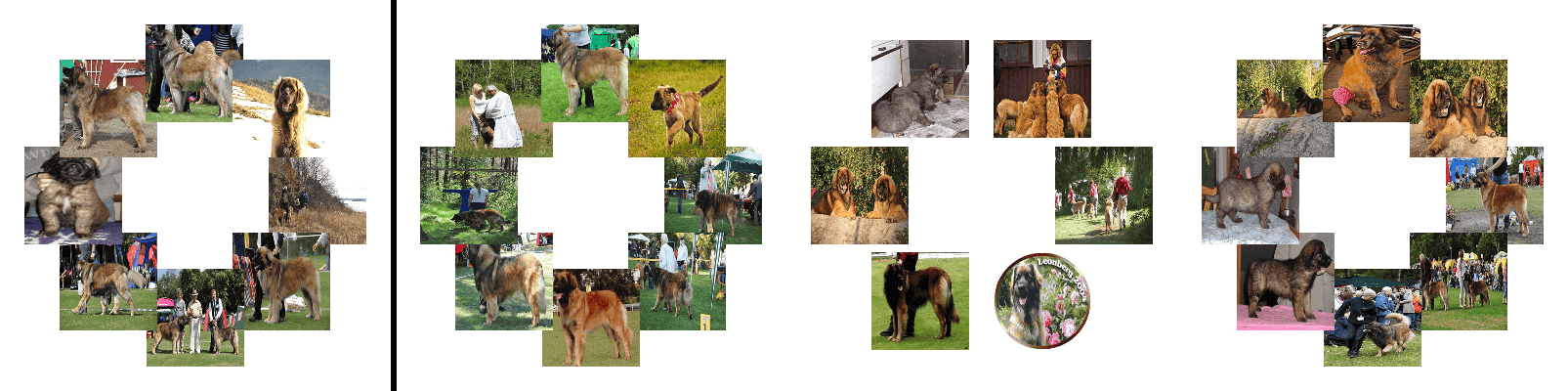} \\
    \includegraphics[width=0.88\linewidth]{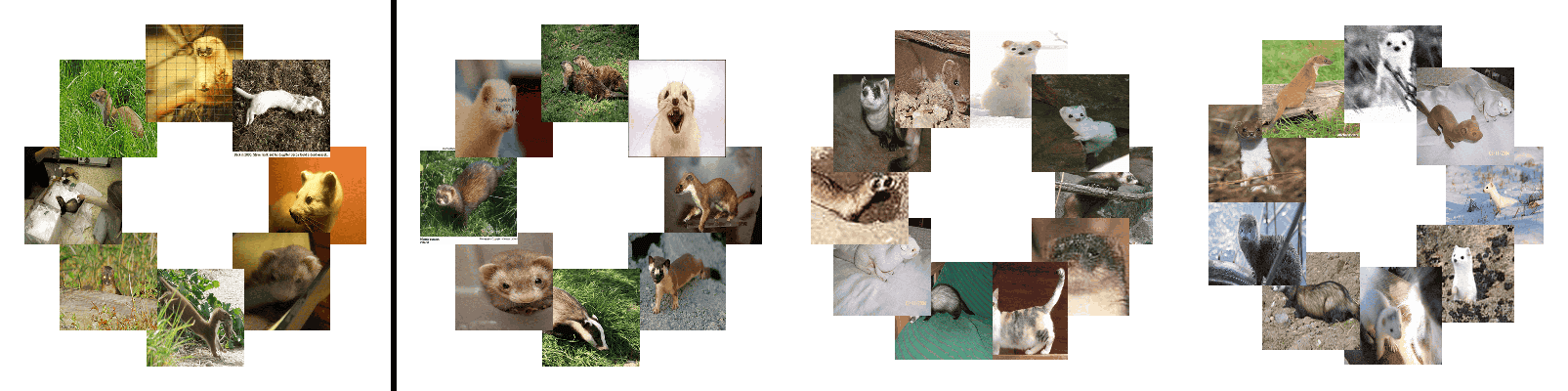} \\ 
    \caption{\textit{\textbf{Additional non-cherry picked retrieval results (1/3):}} Retrieval is performed using model weights, to visualize each model we use the set of all the images that were used to fine-tune the model.}
    \label{fig:retrieval_long1}
\end{figure}

\begin{figure}[t]
    \centering
    \includegraphics[width=0.88\linewidth]{figs/appendix/retrieval/title_row.png} \\ 
    \includegraphics[width=0.88\linewidth]{figs/appendix/retrieval/badger_row.png} \\ 
    \includegraphics[width=0.88\linewidth]{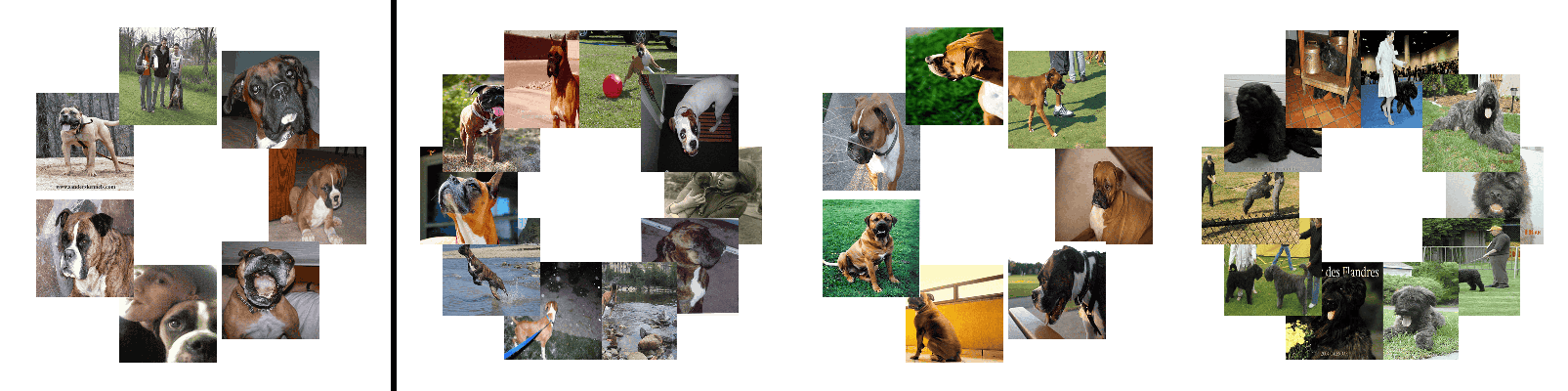} \\ 
    \includegraphics[width=0.88\linewidth]{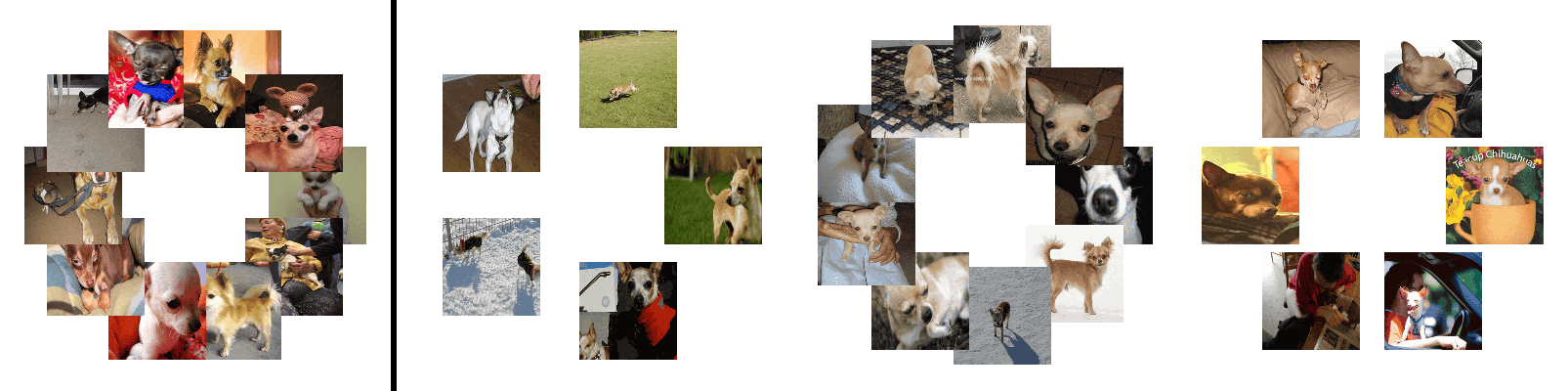} \\ 
    \includegraphics[width=0.88\linewidth]{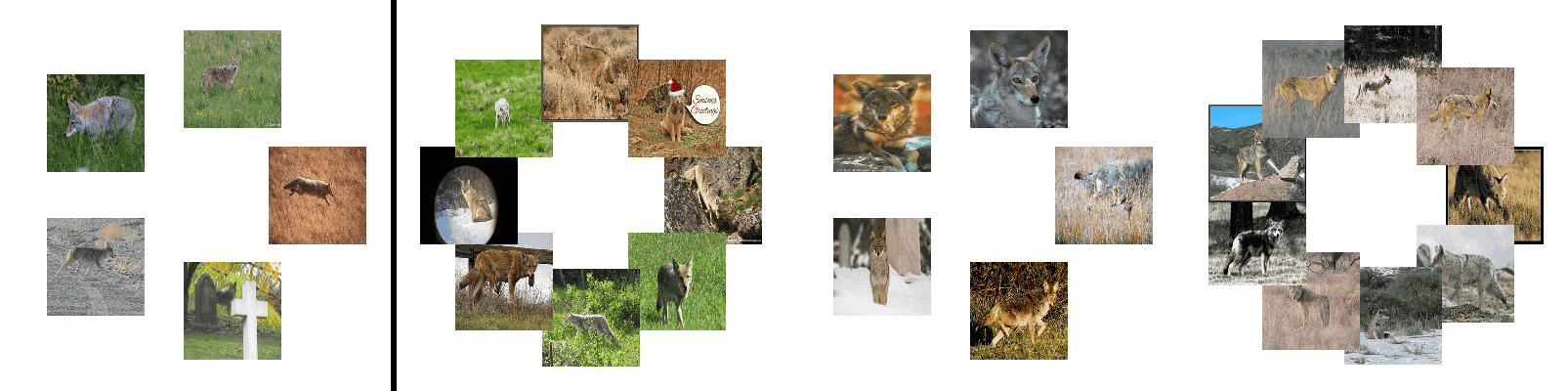} \\ 
    \includegraphics[width=0.88\linewidth]{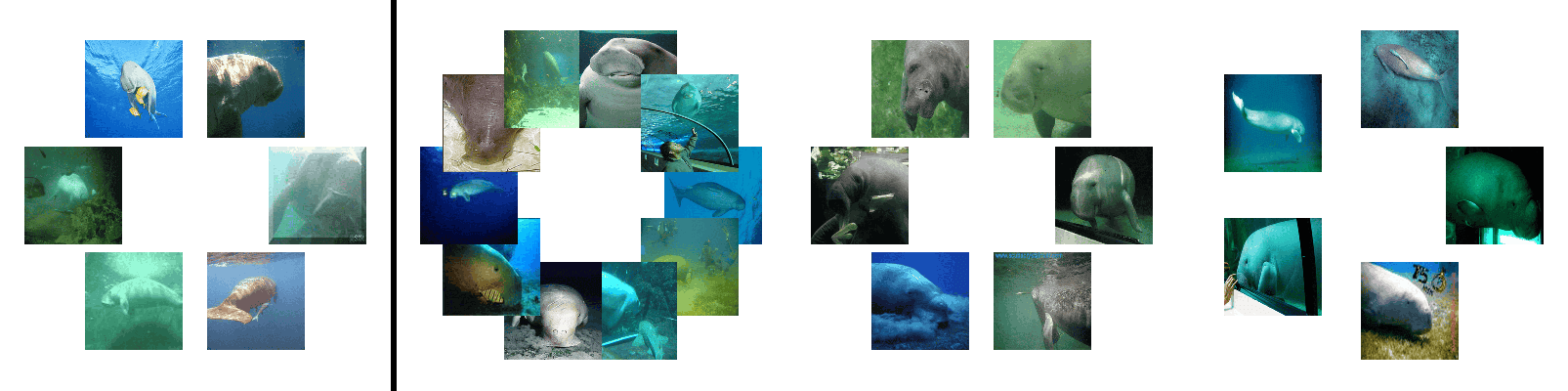} \\ 
    \includegraphics[width=0.88\linewidth]{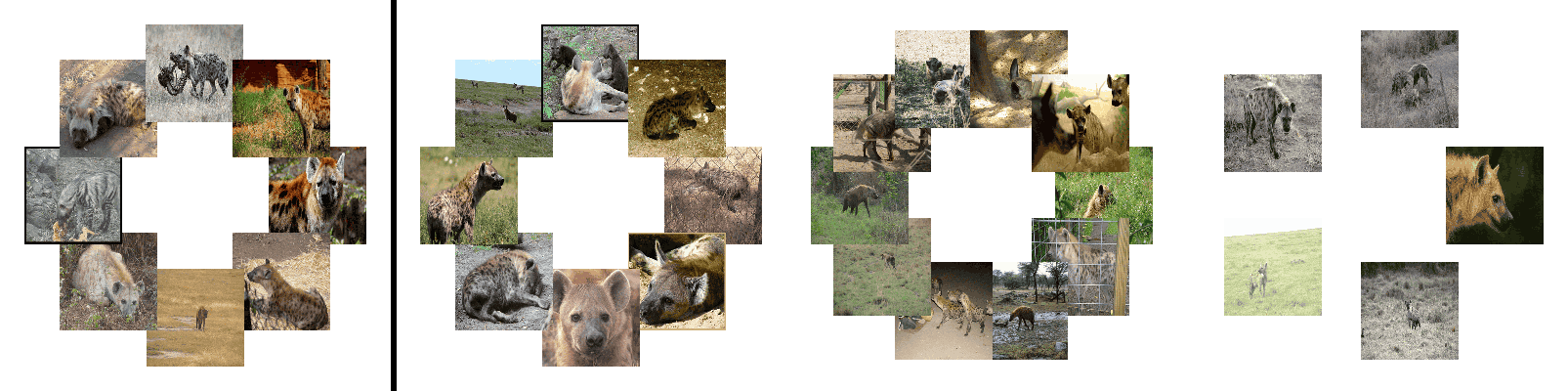} \\ 
    \includegraphics[width=0.88\linewidth]{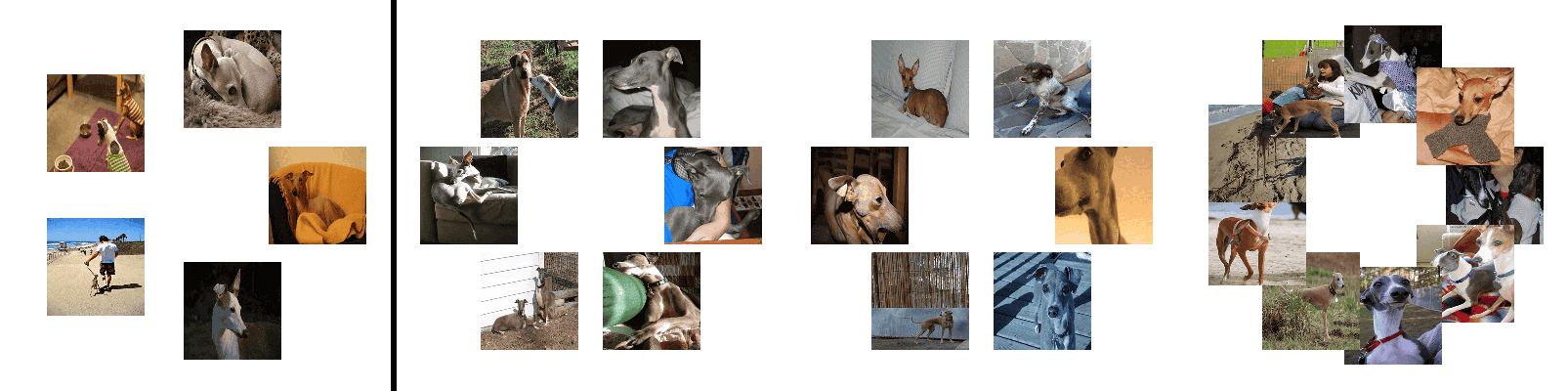} \\ 
    \includegraphics[width=0.88\linewidth]{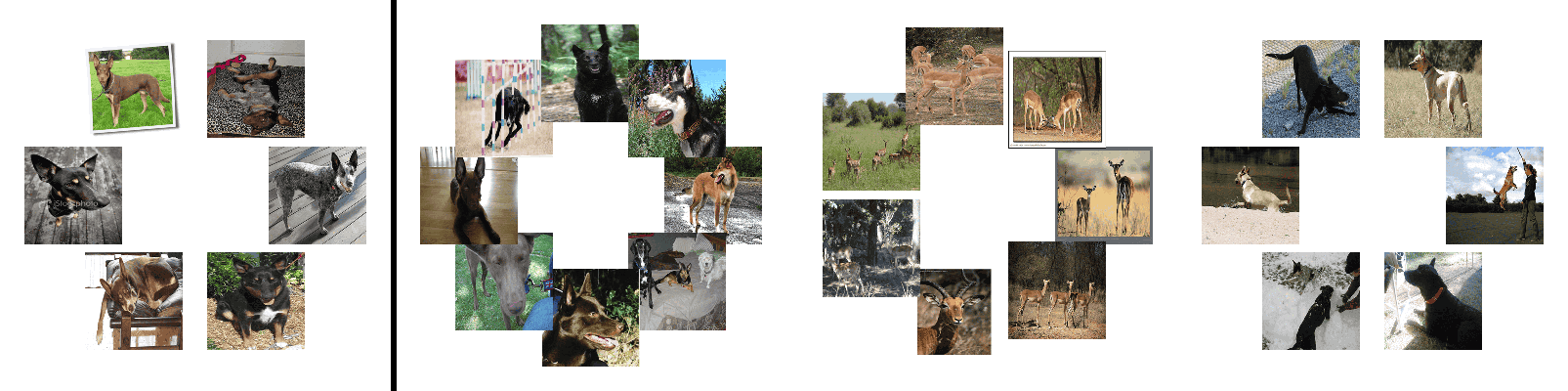} \\ 
    \includegraphics[width=0.88\linewidth]{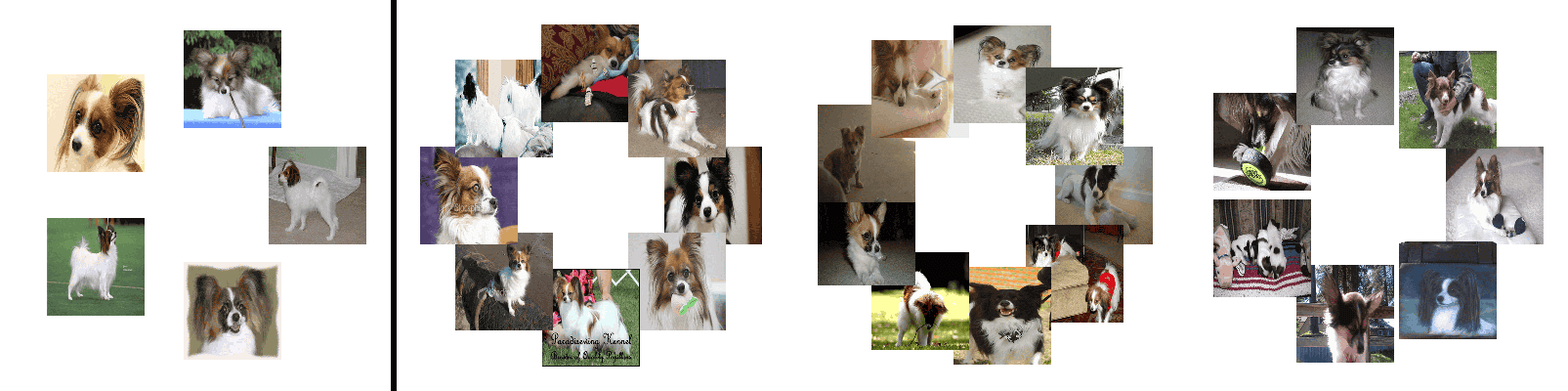} \\
    \includegraphics[width=0.88\linewidth]{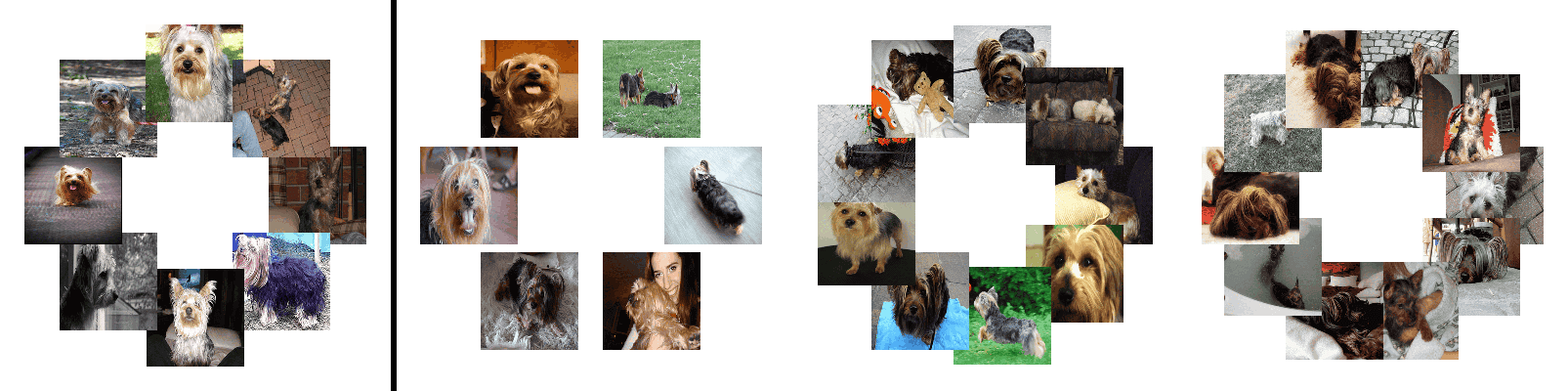} \\ 
    \caption{\textit{\textbf{Additional non-cherry picked retrieval results (2/3):}} Retrieval is performed using model weights, to visualize each model we use the set of all the images that were used to fine-tune the model.}
    \label{fig:retrieval_long2}
\end{figure}

\begin{figure}[t]
    \centering
    \includegraphics[width=0.88\linewidth]{figs/appendix/retrieval/title_row.png} \\ 
    \includegraphics[width=0.88\linewidth]{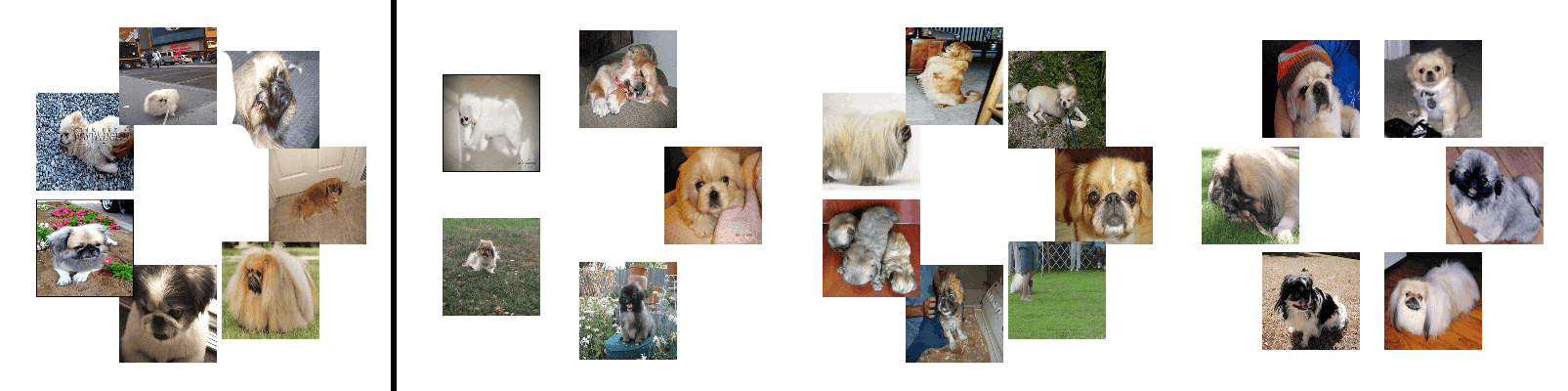} \\ 
    \includegraphics[width=0.88\linewidth]{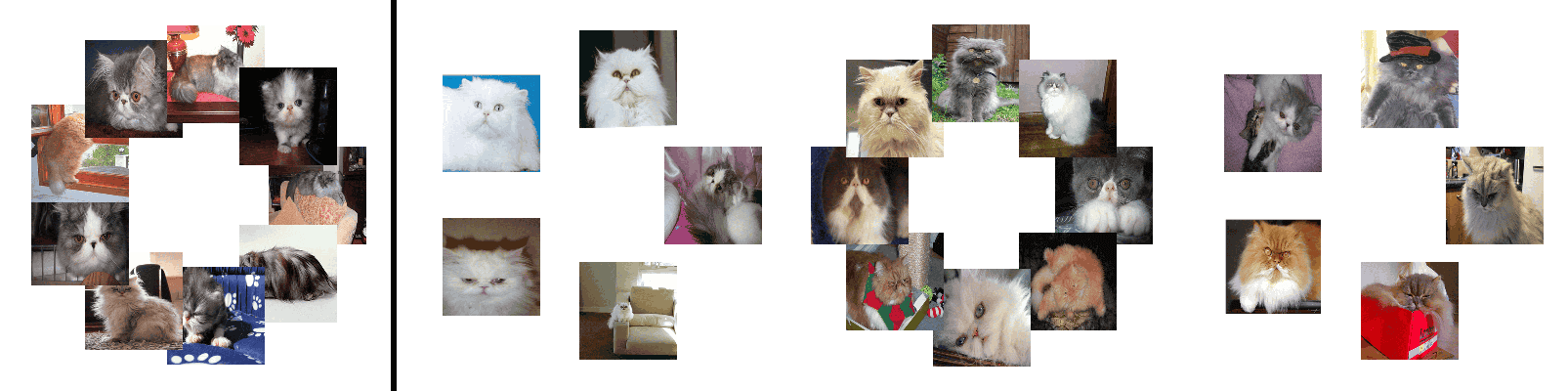} \\ 
    \includegraphics[width=0.88\linewidth]{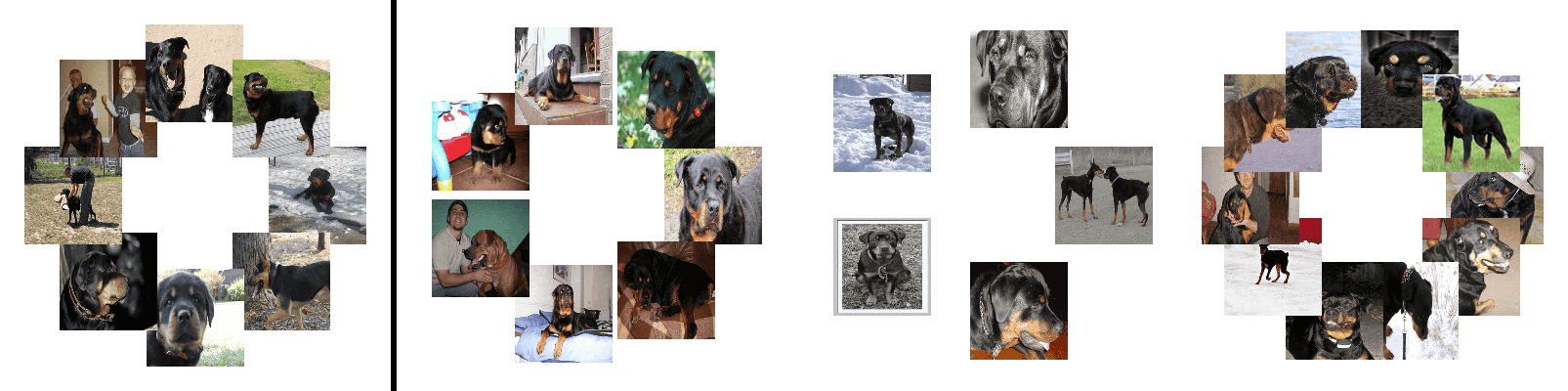} \\ 
    \includegraphics[width=0.88\linewidth]{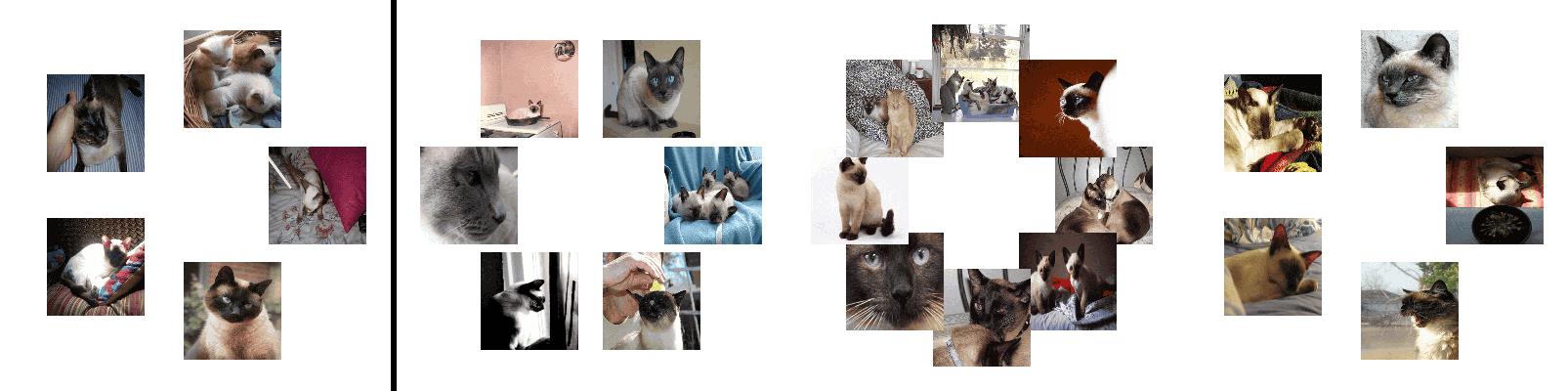} \\ 
    \includegraphics[width=0.88\linewidth]{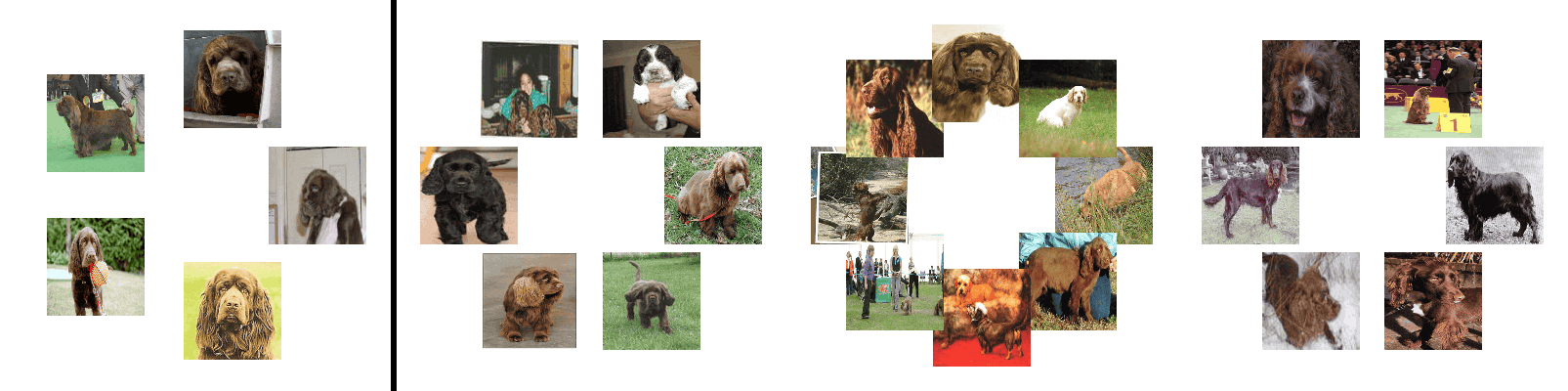} \\ 
    \includegraphics[width=0.88\linewidth]{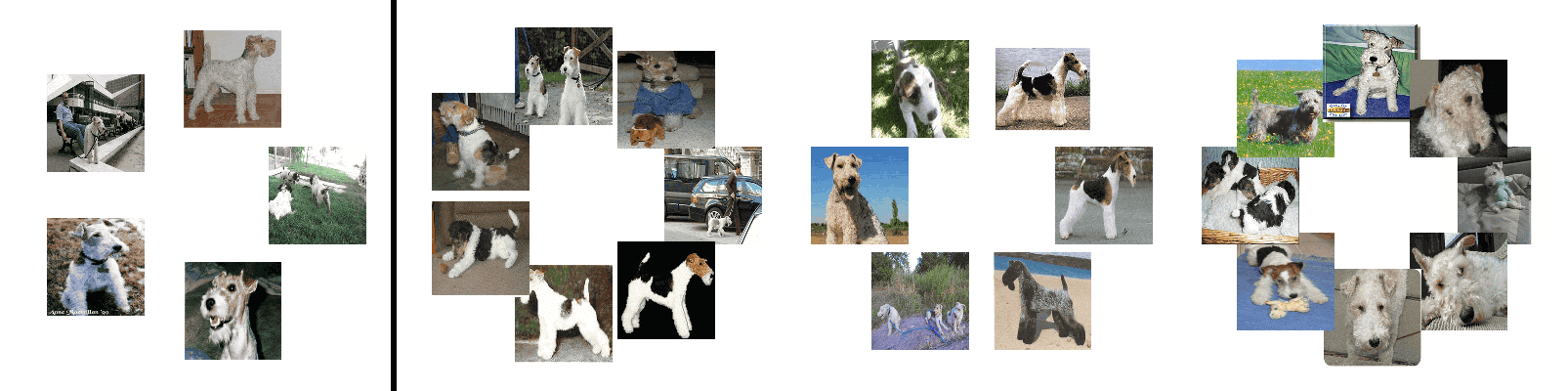} \\
    \includegraphics[width=0.88\linewidth]{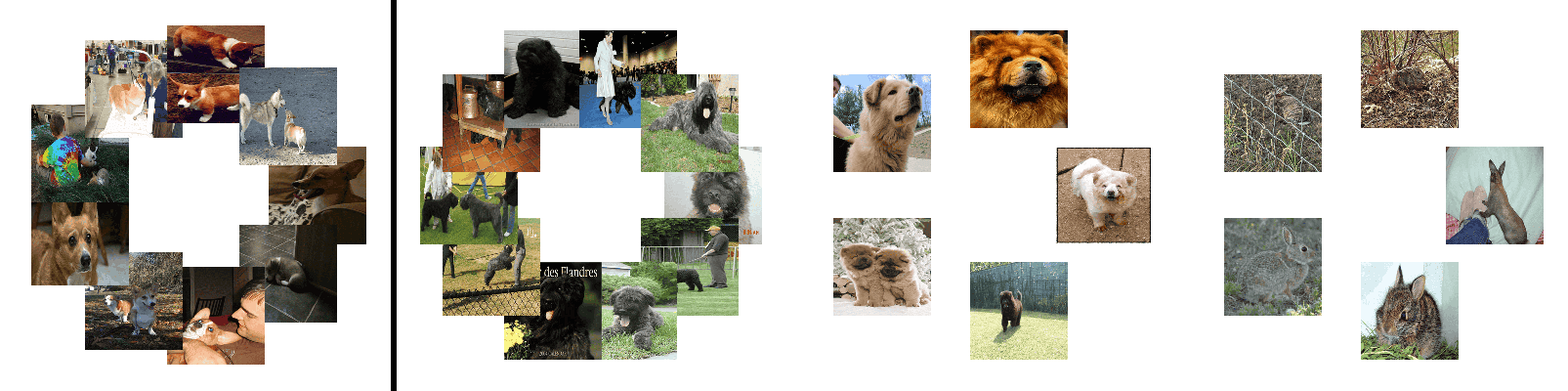} \\ 
    \includegraphics[width=0.88\linewidth]{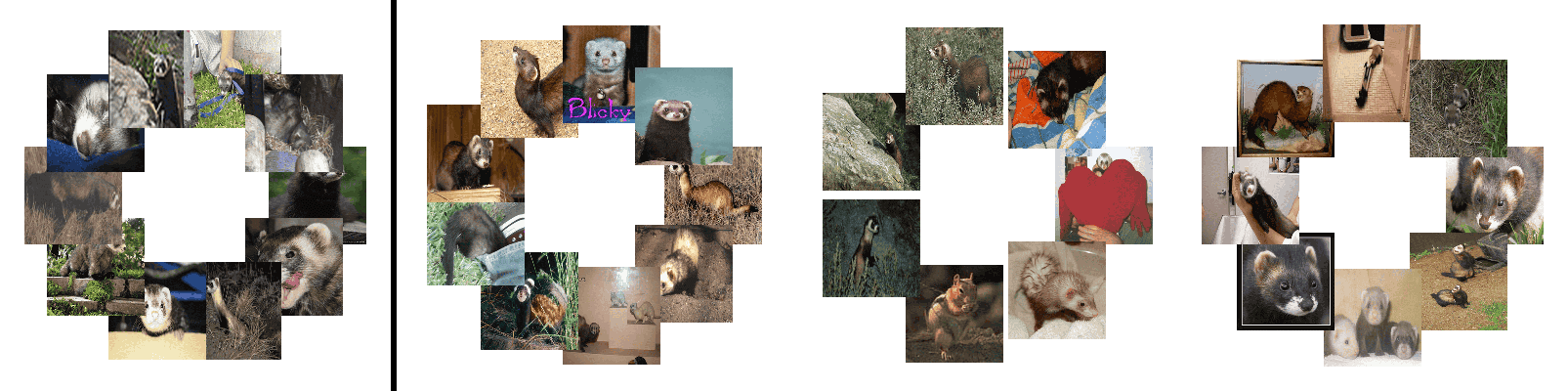} \\ 
    \includegraphics[width=0.88\linewidth]{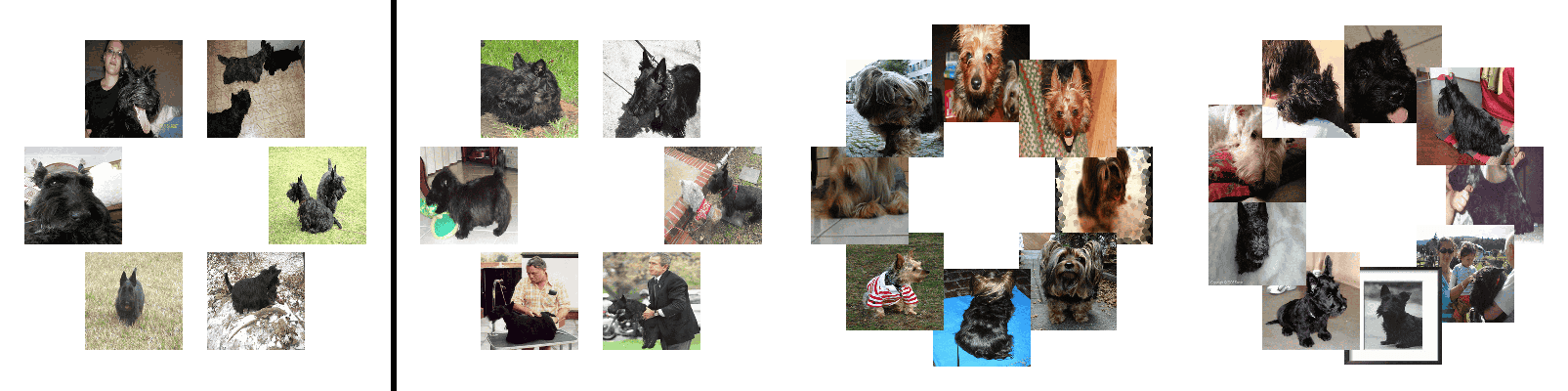} \\ 
    \includegraphics[width=0.88\linewidth]{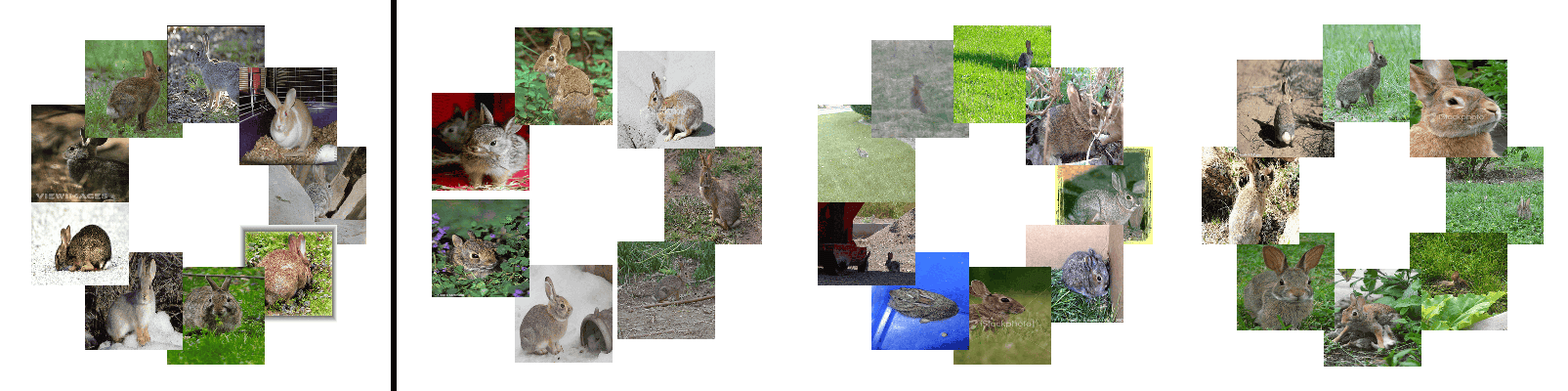} 
    \caption{\textit{\textbf{Additional non-cherry picked retrieval results (3/3):}} Retrieval is performed using model weights, to visualize each model we use the set of all the images that were used to fine-tune the model.}
    \label{fig:retrieval_long3}
\end{figure}

\end{document}